\documentclass[journal,lettersize]{IEEEtran}
\usepackage{amsmath,amsfonts}
\usepackage{array}
\usepackage{textcomp}
\usepackage{stfloats}
\usepackage{caption}
\usepackage{subcaption}
\usepackage{url}
\usepackage{verbatim}
\usepackage{graphicx}
\usepackage{stackengine}
\usepackage{amsthm}
\usepackage{mathtools}
\usepackage{graphicx}
\usepackage{cite}
\usepackage{algorithm}
\usepackage{algpseudocode}
\usepackage[numbers,sort&compress]{natbib}
\usepackage{longtable,tabularx}
\usepackage{booktabs}
\usepackage{balance}
\usepackage{caption}
\usepackage{color}
\usepackage{bm}
\usepackage{xcolor,mdframed}
\usepackage{multirow}
\usepackage{booktabs}
\usepackage{makecell}
\usepackage{array, makecell, booktabs, colortbl, xcolor}

\newtheorem{theorem}{Theorem}
\newtheorem{remark}{Remark}

\algnewcommand{\Inputs}[1]{%
  \State \textbf{Inputs:} 
   \hspace*{0.3em}\parbox[t]{\linewidth}{\raggedright #1}
}
\algnewcommand{\Initialize}[1]{%
  \State \textbf{Initialize:}
   \hspace*{0.3em}\parbox[t]{.8\linewidth}{\raggedright #1}
}
\algnewcommand{\Output}[1]{%
  \State \textbf{Output:}
   \hspace*{0.3em}\parbox[t]{.8\linewidth}{\raggedright #1}
}

\algnewcommand{\OutputSmoothing}[1]{%
  \State \textbf{Output to smoothing:}
   \hspace*{0.3em}\parbox[t]{.8\linewidth}{\raggedright #1}
}

\def\p(#1|#2){p(#1\,|\,#2)}
\def\q(#1|#2){q(#1\,|\,#2)}

\def\BibTeX{{\rm B\kern-.05em{\sc i\kern-.025em b}\kern-.08em
    T\kern-.1667em\lower.7ex\hbox{E}\kern-.125emX}}
\begin{document}
\title{Model Predictive Inferential Control of Neural State-Space Models for Autonomous Vehicle Motion Planning}
\author{Iman Askari, Ali Vaziri, Xuemin Tu, Shen Zeng, and Huazhen Fang
\thanks{This work was sponsored in part by the U.S. Army Research Laboratory under Cooperative Agreements W911NF-22-2-0207 and W911NF-24-2-0163, and in part by the U.S.    National Science Foundation under
Award CMMI-1847651. {\em (Corresponding author: Huazhen Fang.)}}
\thanks{I. Askari, A. Vaziri, and H. Fang are with the Department of Mechanical Engineering, University of Kansas, Lawrence, KS 66045, USA (e-mail: {askari@ku.edu; alivaziri@ku.edu; fang@ku.edu}).}
\thanks{X. Tu is with the Department of Mathematics, University of Kansas, Lawrence, KS 66045, USA (e-mail: {xuemin@ku.edu}).}
\thanks{S. Zeng is with the Department of Electrical \& Systems Engineering, Washington University in St. Louis, St. Louis, MO 63130, USA (e-mail: {s.zeng@wustl.edu}).}}

\maketitle

\vspace{-5em}
\begin{abstract}
Model predictive control (MPC) has proven useful in enabling safe and optimal motion planning for autonomous vehicles. In this paper, we investigate how to achieve MPC-based motion planning when a neural state-space model represents the vehicle dynamics. As the neural state-space model will lead to highly complex, nonlinear and nonconvex optimization landscapes, mainstream gradient-based MPC methods will struggle to provide viable solutions due to heavy computational load. In a departure, we propose the idea of model predictive inferential control (MPIC), which seeks to infer the best control decisions from the control objectives and constraints. Following this idea, we convert the MPC problem for motion planning into a Bayesian state estimation problem. Then, we develop a new implicit particle filtering/smoothing approach to perform the estimation. This approach is implemented as banks of unscented Kalman filters/smoothers and offers high sampling efficiency, fast computation, and estimation accuracy. We evaluate the MPIC approach through a simulation study of autonomous driving in different scenarios, along with an exhaustive comparison with gradient-based MPC. The simulation results show that the MPIC approach has considerable computational efficiency despite complex neural network architectures and the capability to solve large-scale MPC problems for neural state-space models.
\end{abstract}


\section{Introduction}

Autonomous driving is emerging as a transformational technology to reshape the future of transportation and bring tremendous advances in human mobility, traffic efficiency, and roadway safety~\cite{Bagloee:JMT:2016}. A primary challenge to mature this technology is to make autonomous vehicles intelligent decision-makers so that they can drive through traffic skillfully at a level on par with, or better than, human drivers. A key decision-making task is motion planning, which is concerned with identifying the trajectories and maneuvers of the vehicle from a starting configuration to a goal configuration~\cite{Paden:TIV:2016}. Motion plans must ensure safety in traffic, comply with driving customs or laws, and offer passenger comfort across different driving scenarios and traffic conditions.

Motion planning has attracted a large body of research in the past decades. Among the various methods, model predictive control (MPC) has demonstrated significant merits~\cite{Paden:TIV:2016,Claussmann:TITS:2020}. At the heart, MPC enables model-based predictive optimization of motion plans in a receding-horizon fashion to provide important benefits. First, it can take advantage of the vehicle's dynamic model and maneuverability in motion planning to achieve simultaneous path planning and tracking. Second, MPC's innate capability of handling state and input constraints will allow to incorporate practical limits into the planning process. Such limitations typically stem from maneuver limits, safety requirements, and comfort demands. However, while MPC carries promises to compute safe and smooth motion plans, it involves nonlinear constrained optimization, which often brings high computational costs and poor convergence guarantee to global optima~\cite{Claussmann:TITS:2020}. The challenges persist despite recent progress on   MPC  motion planner design and become even stronger in another dimension of growing importance---motion planning based on machine learning models.

Machine learning has risen as a remarkable way for vehicle modeling~\cite{Spielberg:Science:2019,Hewing:TCST:2020}. Its unique strength lies in extracting models from data directly. Given abundant and informative data, such data-driven models based on neural networks or others can effectively capture and predict vehicle dynamics under different and even extreme driving scenarios, while showing robustness against uncertainty of various kinds and non-transparent dynamics. However, machine learning-based vehicle models will be non-trivial for MPC-based motion planning, because of their complex and highly nonlinear nonconvex structure.  While gradient-based optimization solvers have been used to deal with MPC of neural network models, they would be computationally expensive and even brittle in some cases,  e.g., inaccurate initial guesses. The computation 
will be prohibitive if one pursues MPC for planning in the context of reinforcement learning~\cite{Nagabandi:ICRA:2018}.

In this paper, we develop an alternative framework to perform MPC motion planner design when neural network vehicle models are used. The framework, referred to as model predictive inferential control (MPIC), inherits the idea of optimal motion planning from MPC, but pivots away from using gradient-based optimization as solvers. Instead, it undertakes inference-driven decision-making in planning and attempts to estimate the optimal motion plans while using the driving requirements as the evidence. This profound shift allows us to draw on the substantial work on nonlinear estimation in the literature and utilize various estimation methods to do MPC motion planning. Particularly powerful and attractive among those methods is particle filtering. This technique exploits Monte Carlo sampling to achieve accurate estimation despite strong nonlinearities~\cite{Doucet:Springer:2001,Sarkka:Cambridge:2013} and thus provides a leverage to treat the difficulty brought by neural networks to motion planning. However, particle filtering in general is computationally expensive, because it requires intensive sampling, especially when the nonlinearity is sophisticated. This motivates us to investigate a new way to do particle filtering in order to execute the MPIC framework. Our approach, by design, will focus on less intensive but more effective sampling to accelerate computation. %Alexandre:PNAS:2009,Chorin:CAMCS:2010

To summarize at a high level, our study presents two main contributions.
\begin{itemize}

\item {\em The development of the MPIC framework for autonomous vehicle motion planning.} We formulate a motion planner based on incremental-input MPC and show the equivalence of the MPC problem to a receding-horizon Bayesian state estimation problem. Given the translation, we then unveil the forward-filtering/backward-smoothing structure for the  MPIC framework, which runs in receding horizons to estimate control decisions and generate motion plans by inference. 

\item {\em A realization of the MPIC framework based on implicit particle filtering/smoothing.} Our work builds upon the principle of implicit importance sampling, which asserts that, if one manages to find high-probability particles, only a moderate number of them are needed for estimation~\cite{Alexandre:PNAS:2009,Chorin:CAMCS:2010}. Guided by this principle, we develop the implementation of implicit particle filtering/smoothing as banks of unscented Kalman filters/smoothers, which has significant computational efficiency and estimation accuracy. The MPIC realization,  named as the MPIC-X algorithm, is validated via extensive simulations.  

\end{itemize}

The rest of the paper is organized as follows. Section~\ref{sec:related-work} gives a review of the related work. Section~\ref{sec:MP-overview} provides an overview of autonomous vehicle motion planning in structured environments and shows the setup of an incremental-input  MPC planner. Section~\ref{sec:MPIC} reformulates the   MPC problem as a Bayesian state estimation problem and then presents the MPIC framework. Section~\ref{Sec:Kalman-IPF/S} develops the new implicit particle filtering/smoothing approach, which is implemented as banks of unscented Kalman filters/smoothers, to realize the MPIC framework, and Section~\ref{Sec:Discussion}   summarizes the discussion.  Section~\ref{Sec:NumSim} offers simulation results to evaluate the proposed framework and algorithm for motion planning. Then, Section~\ref{sec:Experimental} provides another validation based on a real-world vehicle. Finally, concluding remarks are given in Section~\ref{Sec:Conclusion}.

\section{Related Work}
\label{sec:related-work}

\subsection{Autonomous Vehicle Motion Planning}

A vast literature has formed on motion planning for autonomous vehicles in the past years. While the problem presents itself in different formulations, exact solutions are mostly unavailable, and a diversity of numerical methods have thus flourished. An important category of them, sampling-based planning methods randomly sample across the vehicle configuration space to establish a reachability graph and then find trajectories or paths over the graph~\cite{Paden:TIV:2016}. %{\color{blue} Two popular sampling methods are the probabilistic random maps (PRM)~\cite{Kavraki:TRO:1996}, and rapidly-exploring tree (RRT)~\cite{LaValle:ARR:1998}. The PRM starts by creating an offline roadmap to represent the free configuration space and then uses it to query multiple initial and goal positions through the roadmap using graph search methods, such as A$^{\star}$ search~\cite{Hart:TSSC:1968}.} 
RRT is a popular sampling-based planner, which builds a tree incrementally by random sampling from the start to the goal configurations~\cite{LaValle:ARR:1998}. This method is efficient and provably converges to a suboptimal solution with probability one. More studies have well expanded the scope of RRT for autonomous vehicles, leading to many variants. They include kinodynamic RRT for planning under dynamic constraints~\cite{Frazzoli:JGCD:2002} and closed-loop RRT to deal with closed-loop trajectory prediction~\cite{Kuwata:TCST:2009}. RRT* incorporates the notion of optimality in the tree building process to become asymptotically optimal, though demanding more computation~\cite{Karaman:ICRA:2011,Jeon:ACC:2013}. In general, RRT-based methods are effective in searching through nonconvex, high-dimensional vehicle configuration spaces, for which a theoretical guarantee comes from their probabilistic completeness, i.e., they can find a solution with probability one if it exists. Trajectories that they generate, however, can be jerky and hence require postprocessing for smoothness. 
Probabilistic roadmap (PRM) planners also leverage random sampling to build roadmaps in configuration spaces that allow graph search to find a path between the start and goal configurations~\cite{Kavraki:TRO:1996}.  PRM without differential constraints is probabilistically complete and asymptotically optimal~\cite{Karaman:IJRR:2011}, though it can  be made to accommodate differential constraints in implementation~\cite{Schmerling:ICRA:2015}. Some improved PRM methods, e.g., PRM*, can attain both asymptotic optimality and computational efficiency in~\cite{Karaman:IJRR:2011}.
Tangentially related with RRT and PRM, another sampling-based planner is through particle filtering~\cite{Berntorp:TIV:2019}. Particle filtering is based on sequential Monte Carlo sampling and when applied to motion planning, can sample trajectories based on the driving requirements to achieve better sampling efficiency. As another difference from RRT-based methods, this technique samples the lower-dimensional control input space rather than the vehicle configuration space. 

Despite optimality (or suboptimality) in the probabilistic sense, sampling-based methods can hardly generate truly optimal motion plans in real world because of limited computation and time. Thus, numerical optimization has come as a natural choice to enable optimality-driven motion planning design. This approach can also conveniently include differential or non-holonomic constraints due to vehicle dynamics into the planning process. Some studies have pursued the optimization of trajectories or paths parameterized in certain forms as well as dynamic modification of preplanned paths~\cite{Rucco:CDC:2012,Roesmann:ROBOTIK:2012,Shomin:IROS:2014}. Model predictive control (MPC) has found itself especially suitable for autonomous vehicles operating in dynamic environments and gained considerable interest recently~\cite{Funke:TCST:2017,Nilsson:CEP:2015,Guo:TII:2018,Liu:IV:2017,Claussmann:TITS:2020,Wei:JAS:2023}. This is because it performs dynamic, receding-horizon, and predictive optimization of motion plans. Also, MPC planners can handle the primary concern in driving---safety---by including hard driving constraints resulting from safety or vehicle dynamics, and treat secondary concerns, e.g., passenger comfort and driving ethics, by encoding them into cost functions~\cite{Eiras:TRO:2022,Thornton:TIV:2017}. However, MPC must solve nonlinear, nonconvex constrained optimization problems here, so it may struggle to converge to optimal solutions, or sometimes, feasible solutions even if they exist, and also face high computational costs~\cite{Claussmann:TITS:2020}. To improve convergence, the study in~\cite{Eiras:TRO:2022} develops a two-stage optimization framework, which enforces hard safety or driving constraints in the first stage and then polishes the solution for feasibility and smoothness within the safety bounds in the second stage.  In~\cite{Qian:ITSC:2016}, the motion planning is decomposed into path planning and velocity planning so as to speed up the computation of MPC in implementation. Some recent studies have rallied around using iterative linear quadratic regulation~\cite{Chen:TIV:2019} and differential dynamic programming~\cite{Plancher:IROS:2017,Manchester:CDC:2016,Huang:Allerton:2023}   to approximately solve MPC-based motion planning and trajectory optimization problems.

\subsection{Optimal Control via  Estimation}

The connections between control and estimation have been a fundamental research topic. The seminal work~\cite{Kalman:JBE:1960} by Rudolf Kalman in   1960 elucidates the duality between the linear quadratic regulator and Kalman smoother for linear systems. Even though such exact duality was long considered hardly generalizable, some illuminating studies manage to formulate nonlinear stochastic optimal control and estimation problems dual to each other~\cite{Todorov:CDC:2008,Todorov:PNAS:2009, Kim:TAC:2023}. Optimal control computation often requires great amounts of time and memory, but the control-estimation synergy makes it possible to cast control design as explainable estimation problems which are more tractable to solve. The work in~\cite{Kappen:ML:2012} proposes stochastic optimal control by variational inference based on Kullback-Leibler divergence minimization, which has initiated a line of inquiry. Especially, a special case of the approach~\cite{Kappen:ML:2012} is the path integral control method developed in~\cite{Kappen:PRL:2005}. This relevance has inspired a body of work known as model predictive path integral control (MPPI)~\cite{Williams:JGCD:2017,Williams:TRO:2018}.  
One can  use numerical variational inference methods based on stochastic optimization, Markov Chain Monte Carlo sampling, importance sampling, and others, to compute optimal control for these methods. %{\color{blue} In \cite{Williams:TRO:2018}, a variational-inference approach related to the path-integral control framework is proposed for MPC and solved by iterated importance sampling; however, it uses a random-walk-based search in the state and control spaces, which can be inefficient for high-dimensional problems.} 
In~\cite{Toussaint:ICML:2009,Watson:CoRL:2020}, stochastic optimal control is handled through recursive Bayesian state estimation based on belief propagation, despite only an approximate relationship between the respective control and estimation problems in this case. 
However, for deterministic nonlinear systems, optimal control and MPC problems can find equivalent counterparts in recursive Bayesian state estimation problems~\cite{Stahl:SCL:2011}, and particle filtering provides a useful means to deal with these problems due to its power in handling nonlinearity. 
In our preliminary work~\cite{Askari:ACC:2021,Askari:ACC:2022}, we develop constraint-aware particle filtering/smoothing to perform nonlinear MPC with constraints. Closely tied to optimal control,   reinforcement learning has found some solutions based on  Bayesian estimation~\cite{Levine:arXiv:2018} and variational inference~\cite{Rawlik:RSS:2012}.

\subsection{ MPC of Machine Learning Dynamic Models}

The sweeping successes of machine learning has exponentially driven the study of learning-based MPC for dynamic systems~\cite{Rosolia:ARCRAS:2018,Hewing:ARCRAS:2020}. A  vast literature has formed around a rich set of topics. We narrow the attention here to the specific subject of MPC of neural networks. Neural networks have a history of being used in the data-driven modeling of dynamic systems. Such models are often called neural state-space (NSS) models, and MPCs for them have garnered many applications~\cite{Hewing:ARCRAS:2020}. To date, the solvers are mostly based on gradient-based optimization~\cite{Piche:CSM:2000,Lawrynczuk:Springer:2014,Salzmann:RAL:2023,Bemporad:TAC:2023,Pohlodek:ACC:2023,Cursi:RAL:2021}. Gradient-based search, however, will easily lead 
MPC computation to get stuck in local minima, as NSS models are highly nonlinear and nonconvex.  The study in~\cite{Bunning:PMLR:2021} suggests explicitly constructing neural networks that are convex with respect to the input so as to avoid nonconvex optimization, but such input-convex neural networks would have restricted representation capacity. MPC of NSS models is an essential step in model-based reinforcement learning but faces unaffordable computation when applying gradient-based solutions to models based on deep neural networks. In this context, a simple technique is to randomly generate many candidate control decision sequences and then pick the sequence that leads to the minimum cost after being applied to the NSS model~\cite{Nagabandi:ICRA:2018,Yang:CoRL:2020}. This technique, however, can hardly balance between accuracy and computation.  The problem of  MPC for NSS models hence still remains widely open to new solutions. On a related note,  MPC for systems described by Gaussian processes has recently seen burgeoning studies, e.g., ~\cite{Maiworm:IJRR:2021,Ostafew:IJRR:2016,Hewing:TCST:2020}, and there arise similar challenges.

\subsection{Highlights of Differences of the Study}

The proposed study takes inspirations from many works in the literature, but distinguishes itself from existing research in different dimensions. A summary is as below.

This study and~\cite{Berntorp:TIV:2019} both use particle filtering for autonomous vehicle motion planning. However,  the method in~\cite{Berntorp:TIV:2019} is purely based on Bayesian estimation, leaving its optimality unclear. 
By contrast, our work builds upon an MPC-based formulation and then establishes an equivalent Bayesian estimation problem. This thus infuses optimality into motion planning.  Contrasting~\cite{Berntorp:TIV:2019} further, our study proposes a different, more efficient particle filtering/smoothing method based on implicit importance sampling to run the planning process. 

The work in~\cite{Stahl:SCL:2011} performs MPC without constraints based on particle filtering. This study differs on several aspects. First, we consider incremental-input MPC along with constraints. Second, while conventional state-space models are the focus in~\cite{Stahl:SCL:2011}, what we attempt to deal with is NSS models.  Finally, the solver in~\cite{Stahl:SCL:2011} is bootstrap particle filtering, we show the need to use particle filtering/smoothing and develop a new particle filter/smoother faster and more accurate in estimation.

The work in~\cite{Nagabandi:ICRA:2018,Yang:CoRL:2020} uses random sampling in a forward-simulation manner to find out the best control decisions for the MPC of an NSS model. Easy to implement as it is,
this method requires to exhaustively search through the control space at the cost of computation to achieve just sufficient accuracy. By contrast, the particle filter/smoother enables principled sequential Monte Carlo sampling to gain better accuracy and faster computation. 

The MPPI method in~\cite{Williams:JGCD:2017,Williams:TRO:2018} builds on an information-theoretic formulation of path integral control for stochastic systems, using concepts like free energy, relative entropy, and variational inference. Its implementation relies on a sampling-based iterative search. The search seeks to update control decisions across an entire horizon at each iteration  and often requires a large number of samples, due to the need for calculating path integrals. These factors potentially cause  relatively  expensive computational costs when using NSS models. By comparison, we consider MPC of a deterministic NSS model by Bayesian state estimation rather than variational inference. The recursive structure of Bayesian estimation leads to sequential-in-time computation within a horizon, and further, we design a new particle filter/smoother to use few but highly probable particles. These will accelerate computation significantly to benefit practical implementation. Also, unlike the MPPI method, our approach, by design, explicitly incorporates constraints  and accommodates incremental-input MPC. Meanwhile, it is of our future interest to extend the proposed study to control of stochastic systems. 

Preceding this work, we have developed some preliminary studies~\cite{Askari:ACC:2021,Askari:ACC:2022} about MPC by particle filtering for motion planning. This paper presents two substantial changes. First, we propose to investigate incremental-input MPC of NSS models and convert it into the MPIC framework via Bayesian state estimation. This incremental-input MPC setup is more general, improving constraint satisfaction and smoothness of motion plans. Second, rather than using the bootstrap particle filter and reweighted particle smoother as in~\cite{Askari:ACC:2021,Askari:ACC:2022}, we develop a new implicit particle filter/smoother,   which is structured as banks of unscented Kalman filters/smoothers and designed to draw highly probable particles, to execute the MPIC framework. The resulting method thus has greatly higher sampling and computational efficiency.

\section{Overview of Motion Planning}
\label{sec:MP-overview}
In this section, we first introduce a vehicle model based on neural networks. Then, we focus on the formulation of the motion planning problem by presenting the autonomous driving requirements and subsequently setting up the MPC-based planning problem. 

%\begin{figure}[t]\centering
%\includegraphics[width=0.4\textwidth, trim={0.0cm 0.0cm 0.0cm 0.0cm},clip]{Figures/FrenetCombined.pdf}
%\caption{Coordinate conversion between the global and Frenet coordinates.}
%\label{fig:Global-Frenet-Conversion}
%\end{figure} 

\subsection{Neural State-Space Vehicle Modeling}

%Neural networks have a history of being used in data-driven modeling of dynamic systems since the 1970s~\cite{Schmidhuber:NN:2015}. Recent breakthroughs, including deep learning, have led to significant advances in their representative power, motivating renewed interest in their capacity to describe complex systems. 

  %We leverage neural networks to describe the vehicle dynamics. 
 We consider a vehicle in the global coordinate system. Its state at time $k$ is $\bm x_k = \left[ \begin{matrix} X_k & Y_k & \Phi_k & V_k  \end{matrix} \right]^\top$, where $\left(X_k, Y_k \right)$ is the position, $\Phi_k$ is the heading angle, and $V_k$ is the speed. The vehicle's control input  is $\bm u_k = \left[ \begin{matrix} a_k & \delta_k  \end{matrix} \right]^\top$, where $a_k$ is the acceleration and $\delta_k$ is the steering angle. We use a neural network to capture the vehicle's state evolution in the continuous-time domain:
\begin{align*} %\label{FNN_X_dot}
\dot {\bm x} = f_{\mathrm{NN}}( \bm x, \bm u),
\end{align*}
where $f_{\mathrm{NN}}$ is a feedforward neural network. 
The discrete-time state evolution is then governed by
\begin{align}\label{RNN_X_diff}
 {\bm x}_{k+1} = \hat  f ( \bm x_k, \bm u_k).
\end{align}
We can construct $\hat f$     using different numerical discretization methods. It is often straightforward to use the first-order Euler method, i.e.,
\begin{align*}
  \hat  f ( \bm x_k, \bm u_k) = \bm x_k  + \Delta T \cdot   f_{\mathrm{NN}}( \bm x_k, \bm u_k),
\end{align*}
which is effective if the sampling period $\Delta T$ is small enough. Other methods include the Runge-Kutta schemes, which are more sophisticated and offer better accuracy. Note that alternative ways exist to set up~\eqref{RNN_X_diff}, depending on a vehicle's sensor configuration and data types.  For example, one can train an end-to-end neural network for $\hat f$ directly if $\bm x_k$ and $\bm u_k$ are measured at every time $k$.

As depicted in Fig~\ref{fig:NNS}, the model in~\eqref{RNN_X_diff} is %a recurrent neural network and also 
an NSS model~\cite{Forgione:arXiv:2022}. While taking a concise mathematical form, this model can use multiple hidden layers in the sense of deep learning to extract accurate representations of vehicle dynamics from data. This model also admits different expansions for higher predictive accuracy. For instance,  $ f_{\mathrm{NN}}$ can be designed to use history information to do prediction~\cite{Spielberg:Science:2019}. In this case, its output is still $\bm x_{k+1}$, but its input is $\left \{ \bm x_{k-M:k}, \bm u_{k-M:k} \right\}$, which is the state history over the previous $M$ steps. However, the setup in~\eqref{RNN_X_diff} suffices for our study in this paper, with the proposed results generalizable to more complex NSS vehicle models.

\begin{figure}[t]\centering
\includegraphics[width=0.43\textwidth, trim={0.0cm 0.0cm 0.0cm 0.0cm},clip]{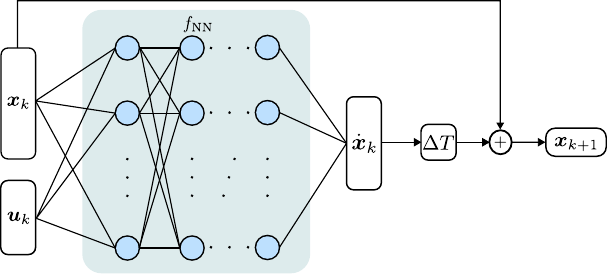}
\caption{Neural state-space model for vehicle dynamics.}
\label{fig:NNS}
\end{figure}

\begin{remark}
Accurate models are key to generating motion plans that ensure vehicle safety and driving performance in dynamic environments.  In general, vehicle modeling can be attained through either employing physical principles or utilizing machine learning techniques.  Physics-based models have mechanistic fidelity to a vehicle's dynamic behavior while presenting themselves in a relatively compact form as a set of nonlinear ordinary differential equations. But they often encounter challenges in capturing the full range of various uncertain effects acting on the vehicle and demand tedious human efforts in model derivation and calibration. Data-driven neural networks have thus gained increasing popularity for vehicle modeling in recent years. Their success results from their universal function approximation properties and powerful descriptive capabilities to capture even highly dynamic vehicle behaviors when given abundant data~\cite{Spielberg:Science:2019,Herman:CIST:2020}. They also allow efficient training and deployment if abundant data are available. However, neural networks are highly nonlinear and nonconvex and require formidable amounts of computation in optimal motion planning. This is the main challenge that motivates the study in this paper. 
\end{remark}

\subsection{Driving  Requirements and Objectives} \label{Constraints}

An autonomous vehicle is expected to operate responsibly within traffic, demonstrating safety, ethics, and predictability in its maneuvers. To this end, it should adhere to some driving requirements and constraints. 
Specifically, the vehicle (referred to as the ego vehicle or EV) must avoid collision with other traffic participants (referred to as obstacle vehicles or OVs), remain within the boundaries of the road, operate in its practical actuation limits,  and follow some nominal guidance (in path, velocity, etc.) generated by a higher-level decision maker.

\textit{Obstacle Avoidance:} Collision avoidance is the top priority for motion planning. This requires the EV to always keep a safe distance from  OVs in its vicinity. We can designate a safety area for each vehicle by bounding the vehicle with certain margins. We denote the safety area as $\mathcal{B}$ indiscriminately for all the EV and OVs for notational simplicity, where $\mathcal{B}$ may take the shape of an ellipsoid or a rectangle. At time $k$ in the planning horizon, the EV's state is $\bm x_k$, and OV $i$'s state is $\bm x_k^{\mathrm{OV},i}$ for $i=1,2,\ldots,N_O$, where $N_O$ is the number of OVs. We then represent the safety areas virtually occupied by the EV and OVs, respectively,  as 
\begin{align*}
\mathcal{S}_k = \mathcal{S} \left(\bm x_k, \mathcal{B} \right), \ \ \mathcal{O}_k^i = \mathcal{O} \left(\bm x_k^{\mathrm{OV},i}, \mathcal{B} \right), \ i=1,2,\ldots, N_O. 
\end{align*}
No collision implies 
\begin{align*}
\mathcal{S}_k \cap \mathcal{O}_k^i  = \emptyset, \ i=1,2,\ldots, N_O. 
\end{align*}
To increase the safety margin, we further impose
\begin{align}\label{Collision-Avoidance-Constraint}
\mathrm{dist} \left(\mathcal{S}_k, \mathcal{O}_k^i  \right)  \geq \underline{d}_O, \ i=1,2,\ldots, N_O, 
\end{align}
where $\mathrm{dist}(\cdot, \cdot) $ is the ordinary distance between two nonempty sets~\cite{Conci:AMSA:2017}, and $\underline{d}_O$ is the required minimum distance between the EV and an OV. % https://arxiv.org/pdf/1808.02574.pdf

\textit{Road Boundary Constraints:} The EV should stay within the lane boundaries for predictable and safe driving. %Its lateral deviation $y_k$ in the Frenet coordinates then must satisfy
Its   deviation from the lane's centerline $\mathcal{LC}$ then must lie  within
\begin{align}\label{boundary-constraint}
- {W_L \over 2}  + \underline{d}_L \leq \mathrm{dist}(\bm x_k, \mathcal{LC}) \leq  {W_L \over 2} - \underline{d}_L.
\end{align} 
In above,  $W_L$ is the lane width, and $\underline{d}_L$ is a restrictive margin. 
\iffalse
with
\begin{align*}
{W_{\mathrm{EV}} \over 2} < \underline{d}_L \leq {W_L \over 2},
\end{align*}
where $W_{\mathrm{EV}}$ is the EV's width. 
\fi

\textit{Vehicle Actuation Limits:} The EV's actuation $\bm u_k$, including the acceleration and steering angle,  is subject to physical limits during the maneuver. %The actuation also affects the level of comfort felt by the passenger. %To account for these factors, 
We thus constrain $\bm u_k$   in motion planning as follows:
\begin{align}\label{control-constraint}
\underline{\bm u} \leq \bm u_k \leq \bar{\bm u}, 
\end{align}
where $\underline{\bm u}$ and $\bar{\bm u}$ are the lower and upper  limits, respectively.  Furthermore, the EV should also bound its ramp-up and ramp-down rates in actuation to meet the need for passenger comfort and improve the smoothness in a computed motion plan. We cast this requirement as a constraint on the incremental control input $\Delta \bm u_k = \bm u_k - \bm u_{k-1}$. Specifically, 
\begin{align}\label{control-increment-constraint}
\Delta \underline{\bm u} \leq \Delta  \bm u_k \leq \Delta  \bar{\bm u}, 
\end{align}
where $\Delta \underline{\bm u}$ and $\Delta \bar {\bm u}$ are the lower and upper control increment limits, respectively.

\textit{Driving Objectives:} While autonomous driving must enforce all the above constraints, it is also crucial that the EV strikes a balance between different objectives, including consistency with the 
 prescribed driving specifications, energy efficiency, and motion comfort, as a human driver often does. A logical approach to this end is to perform  multi-objective minimization of a cost function that penalizes 
\begin{itemize}
\item the tracking error between $\bm x_k$ and the nominal $\bm r_k$ from the higher-level decision maker;
\item the difference between $\bm u_k$ and the prescribed nominal control input $\bm s_k$;
\item the passenger discomfort resulting from $\bm u_k$ and $\Delta \bm u_k$. 
\end{itemize}

Here,  $\bm s_k$ can be selected to meet the need of tracking  $\bm r_k$ or just be zero. The case of $\bm s_k = \bm 0$ implies an intention to minimize actuation efforts in driving. One may also drop this objective if they consider it insignificant in motion planning or find $\bm s_k$ laborious to determine. 

\subsection{Motion Planning Synthesis}

Having laid out the driving requirements and objectives, we are now ready to synthesize the motion planning problem. To begin with, we   express the   constraints in~\eqref{Collision-Avoidance-Constraint}-\eqref{control-increment-constraint} notationally as
\begin{align}\label{collection-constraints}
g_j(\bm x_k, \bm u_k, \Delta \bm u_k) \leq 0,  \    j = 1,\hdots,m,
\end{align}
where $g_j$ is a nonlinear or linear function depending on the specific constraint, and $m$ is the total number of constraints. Consider motion planning within a receding horizon $[k, k+H]$, where $H$ is the horizon length. We have the following  cost function to encode the driving objectives:
\begin{align}\label{cost-function}
& J \left( \bm x_{k:k+H}, \bm u_{k:k+H}, \Delta \bm u_{k:k+H}\right) = \\ & \quad  \quad \quad \quad \quad  \sum_{t=k}^{k+H} \lVert \bm x_{t} - \bm r_{t} \rVert_{\bm R}^2 + \lVert \bm u_t - \bm s_t \rVert_{\bm Q_{\bm u}}^2 + \lVert \Delta \bm u_{t} \rVert_{\bm Q_{\Delta \bm u}}^2, \nonumber
\end{align}
where  $\bm x_{k:k+H} = \left\{ \bm x_k, \ldots, \bm x_{k+H} \right\}$  (the same notation applies to $\bm   u_{k:k+H}$ and $\Delta \bm  u_{k:k+H}$),  $\lVert \cdot \rVert_{\bm S}^2 = \left(\cdot\right)^\top {\bm S}^{-1} \left(\cdot\right)$, and $\bm R$, $\bm Q_{\bm u}$ and $\bm Q_{\Delta \bm u}>0$ are symmetric positive-definite weighting matrices.  One can modify $J$   to have a terminal cost in a distinct form, causing little change to the subsequent development.  
Summarizing the above, an incremental MPC problem is stated as follows for motion planning by the EV:
\begin{subequations}\label{incremental-MPC}
\begin{align}
   \min \quad & J \left( \bm x_{k:k+H}, \bm u_{k:k+H}, \Delta \bm u_{k:k+H}\right), \label{cost-func}\\
\mathrm{s.t.} \quad & \bm x_{t+1} = \hat f( \bm x_{t},\bm u_{t} ) , \label{nn-constraint}\\
& \bm u_{t+1} = \bm u_t + \Delta \bm u_{t+1}, \hspace{0.02cm}
 \footnotesize t = k-1, \ldots, k+H-1, \label{control-update} \\
& g_j(\bm x_t,\bm u_t, \Delta \bm u_t) \leq 0, \quad j=1,\ldots,m, \label{mpc-constraints}\\ 
\nonumber 
&  \footnotesize t = k, \ldots, k+H. 
\end{align}
\end{subequations}

Based on~\eqref{incremental-MPC}, we get an online motion planner.  At every time $k$, it computes the optimal values  $\bm x_{k:k+H}^*$ and $\bm u_{k:k+H}^*$ as the current motion plan,    and then repeats the procedure in a receding-horizon manner as time goes by. 

Conventionally, the primary approach to solving the problem in~\eqref{incremental-MPC} is through gradient-based numerical optimization. However,  the NSS model that appears as a constraint in~\eqref{nn-constraint} turns the optimization problem into a highly nonlinear and nonconvex one, thus posing an immense challenge for computation and practical implementation. Also,  most numerical optimization approaches will result in only local sub-optimal solutions here---although some optimization schemes may help find global optima, they will just make the computation even more expensive~\cite{Du:TIE:2016}.  Breaking away from the tradition, we will examine the above MPC problem from an estimation perspective and formulate the MPIC framework to address it.

\section{Model Predictive Inferential Control}%\label{Sec:MPIC}
  \label{sec:MPIC}
In this section, we will develop the MPIC framework to control  NSS models.  Key to the development is converting the MPC problem in~\eqref{incremental-MPC} into a Bayesian state estimation problem. The resultant MPIC formulation will seek to infer or estimate, rather than optimize as in~\eqref{incremental-MPC}, the best control decisions  (i.e., motion plans) using the driving specifications and constraints (i.e., driving objectives and requirements). With this characteristic, MPIC will lend itself to be implemented by very efficient sampling-based estimation methods, as will be detailed in Section~\ref{Sec:Kalman-IPF/S}.

\subsection{MPC Through the Lens of State  Estimation}

Although~\eqref{incremental-MPC} poses a control problem, we can look at it from an estimation perspective. % and translate it into a Bayesian state estimation problem. 
The central idea lies in treating the nominal driving specifications and constraints as the evidence and then using the evidence to identify what the decision variables should be.  To explain the idea, we begin with setting up a virtual dynamic system:

\begin{equation} \label{Virtual-NSS-model}
\left\{
\begin{aligned} %\label{Compact-Virtual-System-State}
    \bm x_{t+1}&= \hat f( \bm x_t, \bm u_t),\\
      \bm   u_{t+1} &= \bm  u_t + \Delta \bm u_{t+1},\\
	\Delta \bm u_{t+1} &= \bm w_t, \\
       \bm y_{\bm x, t} &= \bm  x_t + \bm  v_{\bm x, t},\\
	\bm y_{\bm u, t} &= \bm  u_t + \bm  v_{\bm u, t},\\
            y_{g, t} &= \sum_{j=1}^m\psi\left( g_j\left(\bm  x_t, \bm  u_t, \Delta \bm u_t \right) \right)  + \varepsilon_t , 
\end{aligned}
\right.
\end{equation}
for $t = k, \ldots, k+H$, where $\bm x_t$, $\bm u_t$ and $\Delta \bm u_t$ are the state variables of the virtual system,   $\bm y_{\bm x, t}$, $\bm y_{\bm u, t}$ and $  y_{g,t}$ are the     measurement variables, and $\bm w_t$, $\bm v_{\bm x,t}$, $\bm v_{\bm u, t}$ and $\varepsilon_t$ are the bounded disturbances. Also, $\psi$ is a barrier function used to measure the constraint satisfaction, which   nominally is
\begin{align}\label{Ideal-Barrier-Funcion}
\psi (x)  = \left\{
\begin{matrix}
0, &  x \leq 0 ,\cr
\infty, & x>0.
\end{matrix}
\right.
\end{align} 
Note that the virtual system replicates the dynamics of the original system as considered  in~\eqref{incremental-MPC}. Further, it introduces $\bm y_{\bm x, t}$, $\bm y_{\bm u, t}$ and $  y_{g,t}$  as the virtual observations of its own behavior---the 
  behavior must correspond to how the MPC formulation in~\eqref{incremental-MPC} steers the original system to behave. This implies that   $\bm y_{\bm x, t}$, $\bm y_{\bm u, t}$ and $  y_{g,t}$ for $t=k,\ldots, k+H$ in an abstract sense should  take the following values, respectively:
\begin{itemize}
\item $\bm y_{\bm x,t} = \bm r_t$ such that $\bm x_t$ follows $\bm r_t$;

\item $\bm y_{\bm u,t} = \bm s_t$ such that $\bm u_t$  follows $ \bm s_t$;

\item $y_{g,t} =   0$ such that all the constraints are satisfied. 
\end{itemize}

For~\eqref{Virtual-NSS-model}, we can pose a moving horizon estimation (MHE) problem to estimate its state, which is given by
\begin{subequations}\label{MHE-problem}
\begin{align}
\min_{} \quad & \sum_{t=k}^{k+H}\lVert \bm v_{\bm x, t} \rVert_{\bm R}^2 + \lVert \bm v_{\bm u, t}\rVert_{\bm Q_{\bm u}}^2 + \lVert \bm w_{t-1}\rVert_{\bm Q_{\Delta \bm u}}^2 + Q_\varepsilon  \varepsilon_t^2,  \\
\mathrm{s.t.} \quad & \bm x_{t+1} = \hat{f}(\bm x_{t},\bm u_{t} ) , \label{mhe-dynamic constraint} \\
& \bm u_{t+1} = \bm u_t + \Delta \bm u_{t + 1}, \\
& \Delta \bm u_{t + 1} = \bm w_t,  \quad   t = k-1, \ldots, k+H-1,\label{mhe-control-constraint} \\
& \bm y_{\bm x,t} = \bm x_t + \bm v_{\bm x, t},  \label{mhe-ref-signal} \\
& \bm y_{\bm u,t} = \bm u_t + \bm v_{\bm u, t}, \label{mhe-control-ref-signal}\\
& y_{g, t}  = \sum_{j=1}^m\psi\left( g_j(\bm  x_t, \bm  u_t, \Delta \bm u_t) \right)  + \varepsilon_t, \\ \nonumber
&  t = k, \ldots, k+H. 
\end{align}
\end{subequations}
where $  Q_\varepsilon >0$. The above problem in~\eqref{MHE-problem} is mathematically equivalent to the MPC problem in~\eqref{incremental-MPC}. The only difference
in their forms is that~\eqref{MHE-problem}  transforms the hard constraints into a penalty term in the cost function as a soft constraint. The relationship between~\eqref{incremental-MPC} and~\eqref{MHE-problem} recalls the duality between MPC and MHE, for which an interested reader is referred to~\cite{Rao:TAC:2003}.  
By solving~\eqref{MHE-problem}, we can obtain the optimal      estimates  of $\bm x_{k:k+H}^*$ and $\bm u_{k:k+H}^*$. But we would face the same computational struggle that afflicts~\eqref{incremental-MPC} if using gradient-based optimization solvers. However,  the MHE formulation would open us to the opportunity of using Bayesian estimation to develop computationally fast solutions. 

% suggests a viable alternative to solving the original MPC problem from an estimation perspective. This

\subsection{MPIC Based on Bayesian State Estimation}

Proceeding forward, we write~\eqref{Virtual-NSS-model} compactly as
\begin{equation}\label{virtual-NSS-compact}
\left\{
\begin{aligned}
 \bar{\bm{x}}_{t+1} &= \bar{\bm f}(\bar{\bm{x}}_t) + \bar{\bm{w}}_t, \\
\bar{\bm y}_t &= 
\bar{\bm h} (\bar{\bm {x}}_t) + \bar{\bm{v}}_t,
\end{aligned}
\right.
\end{equation}
where 
\begin{gather*}
\bar{\bm x}_t  = \left[ 
\begin{matrix}
\bm x_t \cr \bm u_t \cr \Delta \bm u_t
\end{matrix}
\right], 
\; 
\bar{\bm y}_t = \left[ 
\begin{matrix}
\bm y_{\bm x, t} \cr \bm y_{\bm u, t}  \cr y_{g, t}
\end{matrix}
\right],
\; 
\bar{\bm w}_t = \left[ 
\begin{matrix}
\bm 0 \cr \bm w_t \cr \bm w_t
\end{matrix}
\right], 
\; 
\bar{\bm v}_t = \left[ 
\begin{matrix}
\bm v_{\bm x, t} \cr \bm v_{\bm u, t}  \cr \varepsilon_t
\end{matrix}
\right], \\
\bar {\bm f} \left(\bar {\bm x}_t\right) =  \begin{bmatrix}\hat{f}(\bm{x}_{t},\bm{u}_t )\\ \bm{u}_{t}\\  \bm 0
\end{bmatrix}, 
\;
\bar{\bm h}(\bar{\bm x}_t)  = \left[ 
\begin{matrix}
\bm x_t \cr \bm u_t \cr  \displaystyle \sum_{j=1}^m\psi\left( g_j(\bm  x_t, \bm  u_t, \Delta \bm u_t) \right)
\end{matrix}
\right].
\end{gather*}
While~\eqref{Virtual-NSS-model} is subject to bounded disturbances, we   assume~\eqref{virtual-NSS-compact} to be a stochastic system by letting $\bar {\bm w}_t$ and $\bar {\bm v}_t$ be random noise variables.  This implies that $\bar {\bm w}_t$ and $\bar {\bm v}_t$ follow certain probability distributions, and the same holds for $\bar {\bm x}_t$ and $\bar {\bm y}_t$. 

For~\eqref{virtual-NSS-compact}, it is of our interest   to conduct   state estimation of $\bar {\bm x}_t$ for $t=k,\ldots, k+H$ given $\bar {\bm y}_t = \bar {\bm r}_t$ with $
\bar {\bm r}_t = \left[ 
\begin{matrix}
\bm r_t^\top & \bm s_t^\top & 0 
\end{matrix}
\right]^\top$.
This would boil down to considering the posterior probability density function $\p(\bar {\bm x}_{k:k+H} | \bar {\bm y}_{k:k+H} = \bar {\bm r}_{k:k+H}, \bar{\bm x}_{k-1} )$, which captures all the information or knowledge that $\bar {\bm y}_{k:k+H}$ contains about the unknown $\bar {\bm x}_{k:k+H}$. To determine a quantitative  estimate of $\bar {\bm x}_{k:k+H}$, a useful and popular approach is Bayesian maximum a posteriori (MAP) estimation:
\begin{equation}\label{MAP-problem}
\begin{aligned}
\hat{\bar{\bm x}}_{k:k+H}^*  = \arg \max_{\bar {\bm x}_{k:k+H}} \log \p(\bar {\bm x}_{k:k+H} | \bar {\bm y}_{k:k+H} = \bar {\bm r}_{k:k+H}, \bar{\bm x}_{k-1}).
\end{aligned} 
\end{equation} 
We will use $\p(\bar {\bm x}_{k:k+H} | \bar {\bm y}_{k:k+H}, \bar{\bm x}_{k-1})$ by dropping $\bar {\bm r}_{k:k+H}$ in the sequel for notational simplicity. 
The problem in~\eqref{MAP-problem} echoes the MHE problem in~\eqref{MHE-problem} in a way as shown below.

\begin{theorem}\label{theorem-MHE-MAP-equal} Assume that $ \bar{\bm w}_{k:k+H}  $ and $ \bar{\bm v}_{k:k+H}  $ are mutually independent white Gaussian noise processes with 
\begin{align*}
\bar{\bm{w}}_t \sim \mathcal{N}\left(\mathbf{0},  \bar{\bm{Q}}\right), \quad   {\bar{\bm v}}_t \sim \mathcal{N}(\mathbf{0}, \bar{\bm{R}}),
\end{align*}
 where 
\begin{align*}
\bar{\bm{Q}} = 
\left[
\begin{matrix}
\bm{0} &  \bm{0} & \bm{0}\cr  \bm{0}& \bm{Q}_{\Delta \bm u} & \bm{Q}_{\Delta \bm u} \cr \bm{0} & \bm{Q}_{\Delta \bm u} & \bm{Q}_{\Delta \bm u}
\end{matrix}
\right], \quad
\bar{\bm{R}} = \textup{diag}(\bm{R} , \bm{Q}_{\bm u}, Q_\varepsilon).
\end{align*} 
Then, the problems  in~\eqref{MHE-problem} and~\eqref{MAP-problem} will have the same optima. 
\end{theorem}

\begin{proof}
Using Bayes' rule and the Markovian property of~\eqref{Virtual-NSS-model}, we have
\begin{align*}
\p(\bar {\bm x}_{k:k+H} | \bar {\bm y}_{k:k+H}, \bar{\bm x}_{k-1} ) &\propto   
  \prod_{t=k}^{k+H} \p(\bar {\bm y}_t | \bar{\bm{x}}_t) \p(\bar{\bm{x}}_t | \bar{\bm{x}}_{t - 1}),
\end{align*}
%where  $p(\bar{\bm{x}}_k)=1$ as $\bar{\bm{x}}_k$ is known at time $k$. 
which implies
\begin{align*}
&\log \p(\bar {\bm x}_{k:k+H} | \bar {\bm y}_{k:k+H}, \bar{\bm x}_{k-1}) \\
 & \qquad \qquad \qquad \qquad = \sum_{t=k}^{k+H}  \log  \p(\bar {\bm y}_t | \bar{\bm{x}}_t)  
+ \log  \p(\bar {\bm x}_t | \bar{\bm{x}}_{t-1}).
\end{align*}
Given~\eqref{virtual-NSS-compact},  $\p(\bar {\bm y}_t | \bar{\bm{x}}_t)  \sim \mathcal{N}\left(\bar{\bm h} (\bar{\bm {x}}_t),  \bar{\bm{R}}\right)$, implying
\begin{align*}
\log  \p(\bar {\bm y}_t | \bar{\bm{x}}_t) \propto - \left\lVert \bar {\bm v}_t \right\rVert_{\bar{\bm R}}^2.
\end{align*}
Meanwhile, $\p(\bar {\bm x}_t | \bar{\bm{x}}_{t-1})  \sim \mathcal{N}\left(\bar{\bm f} (\bar{\bm {x}}_{t-1}),  \bar{\bm{Q}}\right)$, which is a degenerate Gaussian distribution as $\bar{\bm Q}$  is rank-deficient. By the disintegration theorem~\cite[Theorem 5.4]{Kallenberg:Springer:2021}, $\p(\bar {\bm x}_t | \bar{\bm{x}}_{t-1})$ collapses to a lower-dimensional distribution based on $p(\bm w_{t-1})$, leading to  
\begin{align*}  
\log  \p(\bar {\bm x}_t | \bar{\bm{x}}_{t-1}) \propto - \left\lVert   {\bm w}_{t-1} \right\rVert_{{\bm Q}_{\Delta \bm u}}^2,
\end{align*}
for $t=k,\ldots, k+H$.  

Putting together the above, we find that the cost function in~\eqref{MAP-problem} is given by
\begin{align*}
- \sum_{t=k}^{k+H} \left( \left\lVert \bar {\bm v}_t \right\rVert_{\bar{\bm R}}^2 + \left\lVert   {\bm w}_{t-1} \right\rVert_{{\bm Q}_{\Delta \bm u}}^2 \right),
\end{align*}
which is the opposite of the cost function in~\eqref{MHE-problem}.  The theorem is thus proven. 
\end{proof}

Theorem~\ref{theorem-MHE-MAP-equal} suggests the equivalence of the two estimation problems in a Gaussian setting, by showing that the maxima of~\eqref{MAP-problem} coincide with the minima of~\eqref{MHE-problem}. This key connection allows us to focus on tackling~\eqref{MAP-problem} subsequently. In general, it is not possible to find an analytical solution to~\eqref{MAP-problem}, but we can still develop computable approaches to perform the estimation. The estimation results then will make an approximate solution to~\eqref{MHE-problem} and, in turn, the original MPC problem in~\eqref{incremental-MPC}.

%\subsection{MPIC via Forward-Backward Smoothing}

The  MAP estimation problem in~\eqref{MAP-problem} is known as a smoothing problem in estimation theory, which refers to the reconstruction of the past states using the measurement history. For a stochastic system, this is concerned with computing $\p(\bar {\bm x}_{k:k+H} | \bar {\bm y}_{k:k+H} , \bar{\bm x}_{k-1})$. Since $\bar{\bm x}_{k-1}$ merely denotes the initial condition, we omit it in the sequel for  notational conciseness and without loss of rigor. 
One can break down the computation of $\p(\bar {\bm x}_{k:k+H} | \bar {\bm y}_{k:k+H})$ into two passes~\cite{Sarkka:Cambridge:2013}. The first is the forward filtering pass, which is governed by  
\begin{align}\label{Bayesian-Forward-Filtering-Pass}
\p(\bar {\bm x}_{k:t} | \bar {\bm y}_{k:t} )  & \propto \p( \bar {\bm y}_t |  \bar {\bm x}_t ) \p( \bar {\bm x}_t |  \bar {\bm x}_{t-1} )  \p(\bar {\bm x}_{k:t-1} | \bar {\bm y}_{k:t-1} ),
\end{align}
for  $t=k, \ldots, k+H$.
This relation shows a recursive update from $\p(\bar {\bm x}_{k:t-1} | \bar {\bm y}_{k:t-1})$ to $\p(\bar {\bm x}_{k:t}  | \bar {\bm y}_{k:t})$.  Then, the  filtering distribution $\p(\bar {\bm x}_t | \bar {\bm y}_{k:t} )$ can be obtained by marginalizing out $\bm{\bar x}_{k:t-1}$. The smoothing pass follows the completion of the filtering, which starts from  $ \p(\bar {\bm x}_{k+H} | \bar {\bm y}_{k:k+H})$ and goes all the way back to $ \p(\bar {\bm x}_k | \bar {\bm y}_{k:k+H})$. This is done through  the backward recursion based on
\begin{align} \label{Bayesian-Backward-Smoothing-Pass}\nonumber
 \p(\bar {\bm x}_{t:k+H} | \bar {\bm y}_{k:k+H}) &= \p(\bar {\bm x}_t | \bar {\bm x}_{t+1} , \bar {\bm y}_{k:t} ) \\
 &\quad\quad\cdot \p(\bar {\bm x}_{t+1:k+H} | \bar {\bm y}_{k:k+H}  ),
\end{align}
for $t=k+H-1, \ldots, k$.  Marginalizing 
out $\bar {\bm x}_{t+1:k+H}$  from $\p(\bar {\bm x}_{t:k+H} | \bar {\bm y}_{k:k+H})$ will lead to the smoothing distribution $\p(\bar{\bm x}_{t} | \bar{\bm y}_{k:k+H})$. 
After the smoothing, we can determine $\hat{\bar {\bm x}}_{k}^*$ based on minimum-variance unbiased estimation~\cite{Anderson:Prentice:1979}:  
\begin{align}\label{MVUE}
\hat{\bar {\bm x}}_{k}^* = \int \bar{\bm x}_k  \p(\bar{\bm x}_{k} | \bar{\bm y}_{k:k+H}) d \bar{\bm x}_{k}. 
\end{align} 
This two-pass procedure will run repeatedly at subsequent times in a receding-horizon fashion.

Based on the above, we summarize the MPIC framework in Algorithm~\ref{alg:MPIC}.  Characteristically, the framework exploits  Bayesian inference based on forward-filtering/backward-smoothing to implement MPC, in a shift away from gradient-based numerical optimization. %It builds upon the notion that we can estimate the best control decisions based on the control task specifications and requirements. 
The rich history of research on Bayesian estimation  provides a source of inspirations and insights to handle the MPIC problem for NSS models. Especially, sequential Monte Carlo sampling or 
particle filtering has shown to be powerful for nonlinear state estimation problems, thus holding a strong potential for executing the MPIC framework. Our next focus will be on developing a new particle filter/smoother that presents high computational performance to enable fast MPIC of NSS models. 

\begin{algorithm}[t]
	\caption{The MPIC Framework} \label{alg:MPIC}
	\begin{algorithmic}[1]
 	\State Formulate the MPC problem in~\eqref{incremental-MPC} 
	\State  Set up the virtual system in~\eqref{virtual-NSS-compact}
		\For {$k=1,2,\ldots$} 

\vspace{5pt}

\item[]{\hspace{13pt}\em \color{gray} //  Forward filtering}

			\For {$t=k,k+1,\ldots,k+H$}
				\State Compute $\p(\bar {\bm x}_t | \bar {\bm y}_{k:t}  )$ via~\eqref{Bayesian-Forward-Filtering-Pass}
				 
			\EndFor
 
\vspace{5pt}

\item[]{\hspace{13pt}\em \color{gray} //   Backward smoothing}
 
			\For {$t=k+H-1,k+H-2,\ldots,k$}
				\State Compute $\p(\bar {\bm x}_{t} | \bar {\bm y}_{k:k+H})$ via~\eqref{Bayesian-Backward-Smoothing-Pass}				 
			\EndFor

\vspace{5pt}

\State Compute $\hat{\bar {\bm x}}_{k}^*$ via~\eqref{MVUE}, and apply control $\hat{{\bm u}}_{k}^*$ 
		\EndFor
	\end{algorithmic} 
\end{algorithm}

\section{Implicit Particle Filtering \& Smoothing for MPIC}
\label{Sec:Kalman-IPF/S}

Particle filtering traces its roots to the Monte Carlo simulation. At its core, this approach approximates the posterior probability distribution of a system's state using statistical samples, i.e., particles, and recursively updates the particles using the principle of sequential importance sampling as the state evolves. With the sampling-based nature, particle filtering has great power in approximating complex distributions and searching the state space to make accurate state estimation for highly nonlinear systems. Its capacity is inherently valuable here, and we will leverage particle filtering as well as its sibling, particle smoothing, to approximately implement the MPIC framework.

Particle filtering/smoothing, however, may incur non-trivial computational burden, because one often must use a great number of particles to adequately approximate the target distribution---otherwise, the filtering/smoothing will fail as a majority of the particles might have zero weights in a phenomenon known as particle degeneracy and impoverishment.  In our prior study~\cite{Askari:Auto:2022}, we proposed an implicit particle filtering approach, which shows effectiveness in overcoming the issue. Here, we expand the work in~\cite{Askari:Auto:2022}  by developing implicit particle smoothing to ensure fast computation in the execution of the MPIC framework. For the sake of completeness, we will show a systematic derivation in this section. We will first explain implicit importance sampling and then based on the concept, illustrate how to perform particle filtering/smoothing as banks of Kalman filters/smoothers. Finally, we will present the resulting implementation of MPIC.

\subsection{Implicit Importance Sampling}\label{Sec:IMS}

Let us begin with introducing importance sampling, which is the basis for particle filtering/smoothing. The method is intended to solve a fundamental problem  in Bayesian estimation, which is   evaluating 
\begin{align}\label{Expectation-equation}
\mathbb{E} \left[ \bm g (\bm  x)  \; | \;  \bm y \right] = \int \bm g (\bm  x) \p(\bm x | \bm y ) d \bm x,
\end{align}
where $\bm x$ is a random vector, $\bm y$ is the evidence or observation of $\bm x$,   $\bm g(\cdot)$ is a given function, and $\mathbb{E}$ is the expectation operator. Here, we slightly abuse the notation without causing confusion. The target distribution $\p(\bm x | \bm y)$ is too complex to defy sampling in many cases. So instead we pick an easy-to-sample probability distribution $q(\bm x)$, called the importance distribution, to draw particles $\left\{ \bm x^i, i=1,2,\ldots, N \right\}$. Using the particles, we can obtain an empirical approximation  of~\eqref{Expectation-equation} by
\begin{align}\label{IS-equation}
\mathbb{E} \left[ \bm g (\bm  x)  \; | \;  \bm y \right] \approx \sum_{i=1}^N w^i \bm g (\bm  x^i), \quad w^i \propto \frac{1}{N} \frac{ \p(\bm x^i | \bm y ) }{q(\bm x^i)},
\end{align}
where $w^i$ is the so-called importance weight. 
This technique is known as importance sampling. Despite its utility, low approximation accuracy will result if the particles fall in the low-probability regions of  $\p(\bm x | \bm y)$ when    $q(\bm x)$ does not align well with   $\p( \bm x | \bm y) $. The effects would carry over to particle filtering to underlie the aforementioned issue of particle degeneracy. Consequently, one must use a large, sometimes gigantic, number of particles for satisfactory accuracy at the expense of computational time.

Implicit importance sampling is motivated to remedy the situation. It aims to draw particles such that they lie in the high-probability regions of $\p(\bm x | \bm y)$---when highly probable particles can be found, one can use just fewer of them to achieve not only better approximation accuracy, but also lower computational costs. We introduce a reference probability distribution $p(\bm \xi)$, where the random vector $\bm \xi$ has the same dimension as $\bm x$. Note that $p(\bm \xi)$ must have the same support as $\p(\bm x | \bm y)$, be amenable to sampling, and have well-defined high-probability regions. We now seek to align the high-probability region of $\p(\bm x | \bm y)$ with that of $p(\bm \xi)$. To this end, we define $F(\bm x) = - \log \p(\bm x | \bm y) $ and $\varphi (\bm \xi) =  - \log p(\bm \xi)$, and let
\begin{align}\label{IIS-Equality}
F(\bm x) - \min  F(\bm x) = \varphi (\bm \xi) - \min \varphi (\bm \xi).
\end{align}
Here, $\min  F(\bm x) $ and $\min \varphi(\bm \xi)$ correspond to the highest-probability points of $\p(\bm x | \bm y)$  and $p(\bm \xi)$, respectively. Suppose that a one-to-one mapping $\bm x = \Gamma(\bm \xi)$ solves the algebraic equation~\eqref{IIS-Equality}. Then, we can use the mapping to get the particles $\bm x^i = \Gamma(\bm \xi^i)$, where $\bm x^i$ is highly probably if $\bm \xi^i$ is taken from the highly probable region of $p(\bm \xi)$. 
As shown in~\cite{Askari:Auto:2022}, the importance weight of $\bm x^i$ is
\begin{align*} w^i &\propto \left|\frac{ d \Gamma(\bm \xi^i )  }{d \bm \xi^i}\right|\cdot \exp\left[ -  \min  F(\bm x) + \min \varphi (\bm \xi)  \right].
\end{align*} 
Using the particles along with their weights, we can evaluate $\mathbb{E} \left[ \bm  {g}  (\bm x)  \; | \;  \bm y \right]$ along similar lines in~\eqref{IS-equation}. %The evaluation will be more accurate because the particles are highly probable.

A computational bottleneck for implicit importance sampling is finding $\min F(\bm x)$ and solving~\eqref{IIS-Equality}, especially when $F(\bm x)$ is a nonlinear nonconvex function. However,~\eqref{IIS-Equality}   will admit a closed-form, easy-to-compute expression of $\Gamma(\bm \xi)$ in the Gaussian case. Specifically,
if $\bm \xi \sim \mathcal{N} (\bm 0, \bm I)$ and 
\begin{align*}
\left[
\begin{matrix}
\bm x \cr \bm y
\end{matrix}
\right] \sim \mathcal{N} \left(
\left[
\begin{matrix}
\bm m^{\bm x} \cr \bm m^{\bm y}
\end{matrix}
\right],
\left[
\begin{matrix}
\bm P^{\bm x}  & \bm P^{\bm x \bm y} \cr  \bm \left(\bm P^{\bm x \bm y}\right)^\top  & \bm P^{\bm y}
\end{matrix}
\right]
\right),
\end{align*}
then  $\p(\bm x | \bm y) \sim \mathcal{N}\left(\tilde{\bm m}, \tilde{\bm P} \right)$, where 
\begin{align*}
\tilde {\bm m}  &= \bm m^{\bm x} +  \bm P^{\bm x \bm y}  \left(\bm  P^{\bm y}\right)^{-1} \left( \bm y -  \bm m^{\bm y }\right),\\
\tilde{\bm P} &= \bm P^{\bm x}  - \bm P^{\bm x  \bm y}   \left(\bm  P^{\bm y}\right)^{-1} \left(  \bm P^{\bm x  \bm y}   \right)^\top, \\
\Gamma(\bm \xi)  &= \tilde {\bm m} + \sqrt{\tilde {\bm P}} \bm \xi.
\end{align*}
Based on the explicit form of $\Gamma(\bm \xi)$, we can create the particles $\bm x^i$ and determine their importance weights $\bm w^i$ as
\begin{align*}
\bm x^i = \tilde {\bm m} + \sqrt{\tilde {\bm P}} \bm \xi^i, \quad \bm w^i = {1 \over N},
\end{align*}
where  $\bm \xi^i$ for $ i=1, \ldots, N$  are drawn from the high-probability region of $\mathcal{N}( \bm 0, \bm I)$. 
Then,~\eqref{Expectation-equation} can be evaluated as
\begin{align*}
\mathbb{E} \left[ \bm g (\bm  x)  \; | \;  \bm y \right]  \approx  {1 \over N } \sum_{i=1}^N \bm g (\bm x^i) .
\end{align*}
It is worth highlighting that implicit importance sampling is computationally fast to execute  in the Gaussian case, due to the expedient particle generation. This result is much useful as it will allow us to implement implicit particle filtering/smoothing later as banks of nonlinear Kalman filters/smoothers to enhance the computational efficiency considerably. 

 Next, we exploit implicit importance sampling to develop a new particle filtering/smoothing approach. By design, this approach will allow  to use  few yet highly probable  particles and fast update them in a Gaussian setting. As such, it has high accuracy and computational speed while effectively mitigating the particle degeneracy issue.

\subsection{Implicit Particle Filtering via  a Bank of Kalman Filters} \label{Sec:Kalman-IPF}

%Given the system in~\eqref{virtual-NSS-compact}, we consider the forward Bayesian filtering principle in~\eqref{Bayesian-Forward-Filtering-Pass}.

Given the system in~\eqref{virtual-NSS-compact},     suppose that  we have obtained the particles along with their importance weights
$\left\{ \bar{\bm x}_{k:t-1}^{f,i}, w_{k:t-1}^{f,i}, \,  i=1, \ldots, N \right\}$ at time $t-1$ to form 
an empirical distribution that approximates $\p(\bm{\bar x}_{k:t-1} | \bm{\bar y}_{k:t-1})$:
\begin{align*}
\p(\bar{\bm x}_{k:t-1} | \bar{\bm y}_{k:t-1}) \approx  \sum_{i=1}^N w_{k:t-1}^{f,i} \delta \left(\bar{\bm x}_{k:t-1} - \bar{\bm  x}_{k:t-1}^{f,i} \right),
\end{align*}
where $f$ in the superscripts refers to filtering. Then, as shown in~\cite{Chorin:CAMCS:2010},~\eqref{Bayesian-Forward-Filtering-Pass}   reduces  to 
\begin{align}\label{IPF-Bayesian-filtering}
\p(\bm{ \bar x}_t | \bm{\bar y}_{k:t}) \propto & \, \p(\bm{\bar y}_t | \bm{\bar x}_t)\p(\bm{\bar x} | \bm{\bar x}_{t-1}) \p(\bm{\bar x}_{t-1} | \bm{\bar y}_{k:t-1}).
\end{align}

The key question now is how to identify the particles $\bar {\bm x}_t^{f, i}$ at time $t$ using~\eqref{IPF-Bayesian-filtering} and ensure that  $\bar {\bm x}_t^{f, i}$ will lie in the high-probability region of $\p(\bar{\bm{x}}_{t}  | \bar{\bm{y}}_{k:t})$. The implicit importance sampling principle tells us that we can draw highly probable particles of a reference random vector $\bm \xi_t$ and then map them to $\bar {\bm x}_t $. Following~\eqref{IIS-Equality},  the implicit mapping, denoted as $\Gamma^f: \bm \xi_t \rightarrow \bar {\bm x}_t$, can be made by solving 
\begin{align}\label{IPF-IIS-Equality}
F^f  \left( \bar {\bm x}_t  \right) - \min F^f  \left( \bar {\bm x}_t \right) = \varphi (\bm \xi_t) - \min \varphi (\bm \xi_t),
\end{align}
where $F^f  \left( \bar {\bm x}_t  \right) = - \log \p(\bar{\bm{x}}_t | \bar{\bm{y}}_{k:t})$ and $\varphi (\bm \xi_t) = - \log  p(\bm \xi_t) $.
%Note that in~\eqref{IPF-IIS-Equality}, $\bm{\bar{x}_r}$ represents a variable that we seek to determine for $i = 1, \ldots, N$. Hence, solving~\eqref{IPF-IIS-Equality} for all particles $i$ is computationally intensive since it involves $\min F_i^f(\bm{\bar x}_t)$.

We wish to construct an explicitly expressed  $\Gamma^f(\bm \xi_t)$ to bypass the tedious computation   to solve~\eqref{IPF-IIS-Equality}. This is usually impossible but can be achieved in the Gaussian case, as hinted
 in Section~\ref{Sec:IMS}. 
 %we can obtain an explicit map via a utilitarian Gaussian approximation made locally around $\bm{\bar x}_{t-1}^{f,i}$. 
Our pursuit of the idea goes as below.
Because   
\begin{align*}%\label{y-Prob-Assumption}
\p( \bar {\bm y}_t | \bar{\bm{x}}_{t} , \bar{\bm{x}}_{t-1} ) = \p( \bar {\bm y}_t | \bar{\bm{x}}_{t}),
\end{align*}
it follows from~\eqref{IPF-Bayesian-filtering} and   Bayes' theorem that
\begin{align}\label{IPF-Bayesian-filtering-alternative}
\p(\bar{\bm{x}}_{t} | \bar{\bm{y}}_{k:t}) \propto \p(\bar{\bm{x}}_{t} | \bar{\bm{y}}_t,  \bar{\bm{x}}_{t-1} ) \p( \bar{\bm{y}}_t  |  \bar{\bm{x}}_{t-1} ). % \p(\bar{\bm{x}}_{t-1}  \,  | \, \bar{\bm{y}}_{k:t-1} ).
\end{align}
Let us introduce the following Gaussian distribution approximation locally around  $\bar{\bm{x}}_{t-1}^{f,i}$:
\begin{align*}
\left.
\left[
\begin{matrix}
\bar{\bm{x}}_{t}  \cr \bar{\bm y}_t
\end{matrix}
\right] \, \right| \,
 \bar{\bm{x}}_{t-1} 
&\sim \mathcal{N} \left(
\left[
\begin{matrix}
\bm m^{f,i}_t \cr \hat {\bar{ \bm y}}_t^i
\end{matrix}
\right],
\left[
\begin{matrix}
\bm P^{\bm x, f, i}_t  &    \bm P^{ \bm x \bm y, f, i}_t \cr  \left( \bm P^{ \bm x \bm y, f, i}_t\right)^\top  &  \bm P^{  \bm y, f, i}_t 
\end{matrix}
\right]
\right),
%\\ \bar{\bm{x}}_{t-1}  \,  | \, \bar{\bm{y}}_{k:t-1}   & \sim \mathcal{N} \left( \bar {\bm x}_{t-1}^{f,i}, \bm \Sigma_{t-1}^{f,i} \right). 
\end{align*}
where it is assumed that $ {\bm m}_t^{f,i}$, ${\bm P}_t^{\bm x, f,i}$, $\hat {\bar {\bm y}}_t^{i}$, ${\bm P}_t^{  \bm y, f,i}$, and ${\bm P}_t^{\bm x \bm y, f,i} $ can be constructed from $\bar {\bm x}_{t-1}^{f,i}$. Note that this approximation is local rather than global, without compromising the resulting particle filter's capability to handle non-Gaussian distributions.
It then follows  that
\begin{align}\label{IPF-Posterior-Prob-Update}
\left.  \bar{\bm{x}}_{t}  \, \right| \, \bar{\bm{y}}_{t}, \bar {\bm x}_{t-1}^{f,i} \sim \mathcal{N} \left( \tilde {\bm m}_t^{f, i}, \tilde {\bm P}_t^{\bm x, f,i} \right),
\end{align} 
where
\begin{subequations}\label{IPF-Kalman-Update}
\begin{align}\label{IPF-Kalman-Update-Mean}
\tilde {\bm m}_t^{f, i} &=   {\bm{m}}_{t}^{f,i}  + \bm K_t^{f,i} \left( \bar{\bm y}_t - \hat{\bar{\bm y}}_t^i \right),\\ \label{IPF-Kalman-Update-Cov}
\tilde {\bm P}_t^{\bm x, f, i} &= {\bm P}_t^{\bm x, f, i} -  \bm K_t^{f,i}  \bm P^{  \bm y, f, i}_t \left( \bm K_t^{f,i}\right)^\top,\\ \label{IPF-Kalman-Update-Gain}
 \bm K_t^{f,i}&= \bm P^{ \bm x \bm y, f, i}_t \left(  \bm P^{  \bm y, f, i}_t \right)^{-1}. 
\end{align}
\end{subequations}
%This implies that $\p(\bar{\bm{x}}_{t}^{f,i} | \bar{\bm{y}}_{k:t}) \propto \mathcal{N} \left( \tilde {\bm m}_t^{f, i}, \tilde {\bm P}_t^{f,i} \right)$. 
From~\eqref{IPF-Bayesian-filtering-alternative}-\eqref{IPF-Posterior-Prob-Update},   a local approximation for $\p(\bar{\bm{x}}_{t} | \bar{\bm y}_{k:t} )$ would emerge as  $\p(\bar{\bm{x}}_{t} | \bar{\bm y}_{k:t} )\propto  \mathcal{N} \left( \tilde {\bm m}_t^{f, i}, \tilde {\bm P}_t^{\bm x, f,i} \right)$. 
If further letting $\bm \xi_t  \sim \mathcal{N}(\bm 0, \bm I)$,    an explicit mapping will arise to solve~\eqref{IPF-IIS-Equality}  for $i = 1,\ldots, N$, which is
\begin{align}\label{IPF-particle-perturb}
\bar{\bm{x}}_{t} = \Gamma^f(\bm \xi_t) =  \tilde  {\bm{m}}_{t}^{f,i} + \sqrt{\tilde {\bm P}_t^{\bm x, f, i}}  \bm \xi_t ,
\end{align}
whereby  one can draw a highly probable particle $\bm \xi_t^i$ from $p(\bm \xi_t)$ and then compute $\bar{\bm{x}}_{t}^{f,i} $.

\begin{figure*}
\centering
\vspace{-2mm}
\begin{minipage}{.75\linewidth}
\begin{algorithm}[H]
\caption{The Kalman-IPF Algorithm} \label{alg:Kalman-IPF}
	\begin{algorithmic}[1]
 	\State Initialize the particles $\bar{\bm x}_k^{f,i}$ for $i=1,\ldots, N$ with     $\bm \Sigma_k^{f,i}$ and $w_k^{f,i}$ 
  
			\For {$t= k+1, k+2, \ldots,k+H$}

				\For {$i=1,2,\ldots,N$}

				\State Run Kalman filtering prediction   
\begin{align*}
\left[  {\bm m}_t^{f,i}, {\bm P}_t^{\bm x, f,i}, \hat {\bar {\bm y}}_t^{i}, {\bm P}_t^{  \bm y, f,i},  {\bm P}_t^{\bm x \bm y, f,i}  \right]=\texttt{KalmanPredict}\left( \bar {\bm x}_{t-1}^{f,i},  {\bm \Sigma}_{t-1}^{f,i} \right)
\end{align*}

				 \State  Run  Kalman filtering update  via~\eqref{IPF-Kalman-Update} 
\begin{align*}
\left[ \tilde {\bm m}_t^{f,i},  \tilde {\bm P}_t^{\bm x, f,i} \right]=\texttt{KalmanUpdate}\left( {\bm m}_t^{f,i}, {\bm P}_t^{\bm x, f,i}, \hat {\bar {\bm y}}_t^{i}, {\bm P}_t^{  \bm y, f,i},  {\bm P}_t^{\bm x \bm y, f,i} \right)
\end{align*} 

				\State   Compute $\bar{\bm{x}}_{t}^{f,i}$,  $\bm \Sigma_t^{f,i}$, and $w_t^{f,i}$ via~\eqref{IPF-particle-perturb}-\eqref{IPF-particle-cov} 
 				
				\EndFor
 			\State Do resampling if necessary
			\State Output the filtered state estimate
			\EndFor
  
	\end{algorithmic} 
\end{algorithm}
\end{minipage}
\end{figure*}

As  $\bar{\bm x}_t = \Gamma^f(\bm \xi_t)$ is one-on-one, the importance weight of $\bar{\bm x}_t^{f,i}$ is given by
\begin{align*}%\label{IPF-Importance-Weight}  
w_{t}^{f,i}   \propto w_{t-1}^{f,i} \cdot \left|\frac{ d \Gamma^f(\bm \xi_t^i )  }{d \bm \xi_t^i}\right|   \cdot \exp\left[ -  \min  F^f  \left( \bar {\bm x}_t^{f,i} \right) + \min \varphi (\bm \xi_t^i)  \right]. 
\end{align*} 
This, along with~\eqref{IPF-Bayesian-filtering-alternative}, leads to   the normalized importance weight of $\bar{\bm{x}}_{t}^{f,i}$ computed by
\begin{align} \label{IPF-particle-weight}
w_t^{f,i} = \frac{w_{t-1}^{f,i} \p(\bar {\bm y}_t | \bar{\bm{x}}_{t-1}^{f,i})}{\sum_{j=1}^N w_{t-1}^{f,j} \p(\bar {\bm y}_t | \bar{\bm{x}}_{t-1}^{f,j})}, \ \ \bar {\bm y}_t \, | \, \bar{\bm{x}}_{t-1}^{f,i} \sim \mathcal{N} \left(\hat {\bar{ \bm y}}_t^i, \bm P^{  \bm y, f, i}_t   \right).
\end{align}
As $\bar {\bm x}_t^{f, i}$ is   considered as drawn from $\p(\bar{\bm{x}}_{t} | \bar{\bm{y}}_{k:t})$, it has an associated  covariance ${\bm \Sigma}_t^{f, i} $ that is equal to  $\tilde {\bm P}_t^{\bm x, f, i}$, i.e., 
\begin{align}\label{IPF-particle-cov}
  {\bm \Sigma}_t^{f, i}  = \tilde {\bm P}_t^{\bm x, f, i}.
\end{align}

In the above, what draws our attention is that~\eqref{IPF-Kalman-Update-Mean}-\eqref{IPF-Kalman-Update-Gain} is identical to the well-known Kalman update formulae. This indicates that the particle update can be done here by a nonlinear Kalman filter, and going further, one can use a bank of $N$ nonlinear Kalman filters running in parallel to implement implicit particle filtering. In this implementation, every particle $\bar {\bm x}_{t-1}^{f,i}$ is propagated forward  by a Kalman filter, first to ${\bm m}_t^{f,i}$ by prediction (time-update) and  then to $\tilde {\bm m}_t^{f,i}$ by update (measurement-update)  as in~\eqref{IPF-Kalman-Update};  then, $\tilde  {\bm{m}}_{t}^{f,i}$ will transform to $\bar {\bm x}_t^{f,i}$ with the addition of $ {\bm{\xi}}_{t}^i$. The implementation, referred to as Kalman-IPF, is summarized in Algorithm~\ref{alg:Kalman-IPF}, in which we use {\tt KalmanPredict} and {\tt KalmanUpdate} to represent the prediction-update procedure that is characteristic of Kalman filtering. The Kalman-IPF algorithm will not only inherit the merit of implicit particle filtering in using fewer but higher-quality particles to attain more accurate estimation, but also accelerate the search and determination of such particles. This makes it capable of offering both high accuracy and fast computation.% The   algorithm also allows  to be executed in different ways, depending on  the nonlinear Kalman filter selected to use, and more discussion will be given in Section X.

 %Compared to prior implementations, it  solves~\eqref{IPF-IIS-Equality} efficiently using the Kalman update and thus allows fast computation. 

\begin{remark}
As  mentioned at the beginning of this section,  particle filtering is beset by the issue of particle degeneracy. A commonly used remedy is resampling, which redraws particles based on the discrete empirical distribution formed by the current set of particles to decrease the presence of those particles  with low weights. The Kalman-IPF algorithm is much less vulnerable to the issue due to its capability of finding out highly probable particles. However, we still recommend to include resampling in case particle degeneracy appears.
\end{remark}

\subsection{Implicit Particle Smoothing via  a Bank of Kalman Smoothers}\label{Sec:Kalman-IPS}

Now, we attempt to build implicit particle smoothing upon the notion of implicit importance sampling.  

We start with the backward Bayesian smoothing principle in~\eqref{Bayesian-Backward-Smoothing-Pass}  and assume that $\p(\bar {\bm x}_{t+1} | \bar {\bm y}_{k:k+H})$  is available at time $t+1$ by the approximated empirical distribution of the particles $\left\{ \bar{\bm x}_{t+1}^{s,i}, w_{t+1}^{s,i}, i=1,\ldots,N\right\}$, where the superscript $s$ refers to smoothing.  By~\eqref{Bayesian-Backward-Smoothing-Pass}, when a particle moves from $\bar{\bm x}_{t+1}^{s,i}$  to $\bar{\bm x}_{t}^{s,i}$, we only need to consider 
\begin{align} \label{IPS-Bayesian-Smoothing} 
 \p(\bar {\bm x}_{t} | \bar {\bm y}_{k:k+H}  ) \propto & \, \p(\bar {\bm x}_t | \bar {\bm x}_{t+1}, \bar {\bm y}_{k:t} )   \p(\bar {\bm x}_{t+1} |  \bar {\bm y}_{k:k+H}  ).
\end{align}
Since $\bar {\bm x}_{t}^{s,i}$ is desired to fall into the high-probability regions of $ \p(\bar {\bm x}_{t} | \bar {\bm y}_{k:k+H})$, we want to find a mapping $\Gamma^s: \bm \xi_t \rightarrow \bar {\bm x}_{t}$   to project highly probable values of the reference sample $\bm \xi_t^i$ to highly probable values of $\bar {\bm x}_{t}^{s,i}$.   Based on implicit importance sampling, we let
\begin{align}\label{IPS-IIS-Equality} 
F^s  \left( \bar {\bm x}_t  \right) - \min F^s \left( \bar {\bm x}_t \right) = \varphi (\bm \xi_t) - \min \varphi (\bm \xi_t),
\end{align}
where  $F^s  \left( \bar {\bm x}_t \right) = - \log \p(\bar{\bm{x}}_{t}  | \bar{\bm{y}}_{k:k+H} )$, and $\varphi (\bm \xi_t)$ follows the same definition in~\eqref{IPF-IIS-Equality}. %{\color{blue} Again, we note that $\bm{\bar x}_t$ in~\eqref{IPS-IIS-Equality} represents a variable that we seek to determine for all particles $i = 1, \ldots, N$}. % and reused for notational simplicity and without loss of clarity. 
As before, an explicit expression for $\bar{\bm x}_t   = \Gamma^s (\bm \xi_t)$ will be sought after to solve~\eqref{IPS-IIS-Equality}. To achieve this, we impose the following local Gaussian approximations: % around the local samples $\bm{\bar x}_t^{f,i}$, $\bm{\bar x}_{t+1}^{f,i}$, and $\bm{\bar x}_{t+1}^{s,i}$:}
\begin{align*}
\left.
\left[
\begin{matrix}
\bar{\bm{x}}_{t}  \cr \bar{\bm{x}}_{t+1} 
\end{matrix}
\right] \, \right| \,
 \bar{\bm{y}}_{k:t}
&\sim \mathcal{N} \left(
\left[
\begin{matrix}
\bar{\bm x}^{f,i}_t \cr  {\bm m}^{f,i}_{t+1}
\end{matrix}
\right],
\left[
\begin{matrix}
\bm \Sigma^{f, i}_t  &    \bm P^{ \bm x  , f, i}_{t,t+1} \cr  \left( \bm P^{ \bm x, f, i}_{t,t+1}\right)^\top  &  \bm P^{\bm x, f, i}_{t+1} 
\end{matrix}
\right]
\right),\\
\left. \bar{\bm{x}}_{t+1}  \, \right| \,  \bar{\bm{y}}_{k:k+H} &\sim 
\mathcal{N} \left(\bar{\bm{x}}_{t+1}^{s,i}, \bm \Sigma_{t+1}^{s,i} \right).
\end{align*}
Inserting the above Gaussian distributions into~\eqref{IPS-Bayesian-Smoothing}, we have
\begin{align}\label{IPS-Posterior-Prob-Update}
\bar {\bm x}_t \, | \, \bar {\bm y}_{k:k+H} \sim \mathcal{N} \left(
\tilde {\bm m}_t^{s,i}, \tilde {\bm P}_t^{\bm x, s,i}
\right),
\end{align}
where
\begin{subequations}\label{IPS-Kalman-Update}
\begin{align}\label{IPS-Kalman-Update-Mean}
\tilde {\bm m}_t^{s, i} &=   \bar{\bm{x}}_{t}^{f,i}  +  \bm K_t^{s,i}  \left( \bar{\bm{x}}_{t+1}^{s,i} -  {\bm{m}}_{t+1}^{f,i} \right),\\ \label{IPS-Kalman-Update-Cov}
\tilde {\bm P}_t^{\bm x, s, i} &= {\bm \Sigma}_t^{f, i} + \bm K_t^{s,i} \left(  \bm \Sigma_{t+1}^{s,i} -  \bm P_{t+1}^{\bm x, f,i}  \right) \left(\bm K_t^{s,i}\right)^\top,\\
\bm K_t^{s,i} &= \bm P^{ \bm x  , f, i}_{t,t+1}   \left(  \bm P^{ \bm x  , f, i}_{t+1}  \right)^{-1} . 
\end{align}
\end{subequations}
Then, combining~\eqref{IPS-Posterior-Prob-Update} with~\eqref{IPS-IIS-Equality} readily gives
\begin{align}\label{IPS-particle-perturb}
\bar{\bm{x}}_{t} = \Gamma^s(\bm \xi_t ) =  \tilde  {\bm{m}}_{t}^{s,i} + \sqrt{\tilde {\bm P}_t^{\bm x, s, i}}  \bm \xi_t ,
\end{align}
when $\bm \xi_t \sim \mathcal{N}(\bm 0, \bm I)$. Drawing a highly probably particle $\bm \xi_t^i$ for all $i=1,\ldots,N$, we can quickly compute $\bar{\bm x}_t ^{s,i}$ using~\eqref{IPS-particle-perturb} and obtain the associated covariance as
\begin{align}
\bm \Sigma_t^{s,i} = \tilde {\bm P}_t^{\bm x, s, i}. 
\end{align}
For the backward smoothing, the importance weight of $\bar {\bm x}_t^{s,i}$ is
\begin{align*}
w_t^{s,i} \propto \left|\frac{ d \Gamma^s(\bm \xi_t^i )  }{d \bm \xi_t^i}\right|   \cdot \exp\left[ -  \min  F^s  \left( \bar {\bm x}_t^{s,i} \right) + \min \varphi (\bm \xi_t)  \right], 
\end{align*} 
which suggests that $w_t^{s,i} = w_t^{s,j}$ for any $i$ and $j$ due to~\eqref{IPS-particle-perturb}. Hence, the normalized weight is
\begin{align}\label{IPS-weight}
w_t^{s,i} = {1 \over N}. 
\end{align}
The equal weights here indicate that all the particles have the same importance. Hence, the backward smoothing pass is free from the particle degeneracy issue, making resampling unnecessary in this pass. 
\begin{figure*}
\centering
\vspace{-2mm}
\begin{minipage}{.85\linewidth}
\begin{algorithm}[H]
	\caption{The Kalman-IPS Algorithm} \label{alg:Kalman-IPS}
	\begin{algorithmic}[1]
 	\State Obtain the particles $\bar{\bm x}_{k+H}^{s,i}$ for $i=1,\ldots, N$ with     $\bm \Sigma_{k+H}^{s,i}$ and $w_{k+H}^{s,i}$  from executing the Kalman-IPF algorithm
  
			\For {$t= k+H, k+H-1, \ldots,k$}

				\For {$i=1,2,\ldots,N$} 

				 \State  Run  Kalman smoothing update  via~\eqref{IPS-Kalman-Update}
				%\quad  $\Scale[1]{  \left[ \tilde {\bm m}_t^{s,i},  \tilde {\bm P}_t^{s,i} \right]={\tt KalmanSmooth}\left( \bar{\bm{x}}_{t+1}^{s,i}, \bm \Sigma_{t+1}^{s,i},  \bar {\bm x}_t^{f,i}, {\bm \Sigma}_t^{ f,i}, \bm m_{t+1}^{f,i}, \bm {\bm P}_{t+1}^{\bm x, f,i}, \bm {\bm P}_{t,t+1}^{\bm x, f,i} \right) }$ 
\begin{align*}
\left[ \tilde {\bm m}_t^{s,i},  \tilde {\bm P}_t^{s,i} \right]=\texttt{RTSKalmanSmooth}\left( \bar{\bm{x}}_{t+1}^{s,i}, \bm \Sigma_{t+1}^{s,i},  \bar {\bm x}_t^{f,i}, {\bm \Sigma}_t^{ f,i}, \bm m_{t+1}^{f,i}, \bm {\bm P}_{t+1}^{\bm x, f,i}, \bm {\bm P}_{t,t+1}^{\bm x, f,i} \right)
\end{align*}

%\vspace{5pt}

				\State   Compute $\bar{\bm{x}}_{t}^{s,i}$,  $\bm \Sigma_t^{s,i}$, and $w_t^{s,i}$ via~\eqref{IPS-particle-perturb}-\eqref{IPS-weight}
%\begin{align*} \bar{\bm{x}}_{t}^{f,i} =   \tilde  {\bm{m}}_{t}^{f,i} + \sqrt{\tilde {\bm P}_t^{f, i}}  \bm \xi_t^i, \ \ {\bm \Sigma}_t^{f, i}  = \tilde {\bm P}_t^{f, i} \end{align*}
 				
				\EndFor
% 			\State Do resampling if necessary
			\State Output the smoothed state estimate
			\EndFor
  
	\end{algorithmic} 
\end{algorithm}
\end{minipage}
\end{figure*}

\begin{figure}
\centering
\vspace{-2mm}
\begin{minipage}{\linewidth}
\begin{algorithm}[H]
\fontsize{9}{10}
  \caption{The Unscented Transform ($\mathcal{UT}$)}\label{alg:UncentedTransform}
  \begin{algorithmic}[1]
\Inputs{$\bm y = \gamma (\bm x)+\bm w$, $\bm{m}^{\bm x}$, $\bm P^{\bm x}$, $\bm Q$} % with $\bm x \in \mathbb{R}^{n}$ $\lambda = \alpha^2(n+\kappa)-n$, $\kappa = 3-n$}
\State Generate the sigma points  
\begin{align*}
\bm{x}^0 &= \bm m^{\bm x}\\
\bm{x}^i &= \bm m^{\bm x} + \sqrt{\alpha}\left[\sqrt{\bm P^{\bm x}}\right]_i \, , \, i = 1,2,\ldots,n\\
\bm{x}^{i+n} &= \bm m^{\bm x} - \sqrt{\alpha}\left[\sqrt{\bm P^{\bm x}}\right]_i \, , \, i = 1,2,\ldots,n
\end{align*}
\State Pass the sigma points through $\gamma(\cdot)$
\begin{align*}
\bm y^i &= \gamma(\bm x^i)\, , \, i=0,1,\ldots,2n 
\end{align*}
\iffalse
\State Assigning weights ($\beta = 2$):
\Statex $w_0^{(m)} = \frac{\lambda}{n_{\bar x}+\lambda}$
\vspace*{0.2em}
\Statex $w_0^{(c)} = \frac{\lambda}{n_{\bar x}+\lambda}+(1-\alpha^2+\beta)$
\vspace*{0.2em}
\Statex $w_j^{(c)} = w_j^{(c)} = \frac{\lambda}{2(n_{\bar x}+\lambda)},\; \: j = 0,1,\hdots,2n_{\bar x}$
\fi
\State Compute $\bm m^{\bm y}$,  $\bm P^{\bm y}$, and  $\bm P^{\bm x \bm y}$
\begin{align*}
\bm m^{\bm y} &= \sum_{i=0}^{2n} W_m^i \bm  y^i\\
\bm P^{\bm y} &= \sum_{i=0}^{2n} W_c^i \left(\bm  y^i - \bm m^{\bm y} \right)\left(\bm  y^i - \bm m^{\bm y} \right)^\top + \bm Q\\
\bm P^{\bm x \bm y} &= \sum_{i=0}^{2n} W_c^i \left(\bm  x^i - \bm m^{\bm x} \right)\left(\bm  y^i - \bm m^{\bm y} \right)^\top
\end{align*} 
\Output{$\bm m^{\bm y}$, $\bm P^{\bm y}$, $\bm P^{\bm x\bm y}$}
\end{algorithmic}
\end{algorithm}
\vspace{-4mm}
{\footnotesize 
%\textsuperscript{*}For a matrix $\bm A$, $\left[\bm A\right]_i$ is its $i$-th column.\\
\textsuperscript{*}See~\cite{Fang:JAC:2018} for the definitions of $\alpha$, $W_m^i$, and $W_c^i$.}
\end{minipage}
\end{figure}

The backward update in~\eqref{IPS-Posterior-Prob-Update}-\eqref{IPS-Kalman-Update} is identical to the Rauch-Tung-Striebel (RTS) Kalman smoothing. This inspires  using  a bank of nonlinear Kalman smoothers to implement implicit particle smoothing, with the smoothers running concurrently to update the individual particles. We summarize the procedure in Algorithm~\ref{alg:Kalman-IPS}  and call it Kalman-IPS. In the algorithm, {\tt RTSKalmanSmooth} represents the RTS backward update. %It   can be achieved in different ways as in the case of the Kalman-IPF algorithm, and a specific way will be shown in Section~\ref{}. 

\begin{figure*}
\centering
\vspace{-2mm}
\begin{minipage}{.85\linewidth}
\begin{algorithm}[H]
	\caption{The MPIC-X Algorithm} \label{alg:MPIC-X}
	\begin{algorithmic}[1]
 	\State Formulate the MPC problem in~\eqref{incremental-MPC} 
	\State  Set up the virtual system in~\eqref{Virtual-NSS-model}
		\For {$k=1,2,\ldots$} 

\vspace{5pt}

\item[]{\hspace{13pt} \em \color{gray}//  Forward filtering by Kalman-IPF}

\State Initialize the filtering particles $\bar{\bm x}_k^{f,i}$ for $i=1,\ldots, N$ with     $\bm \Sigma_k^{f,i}$ and $w_k^{f,i}$ 
			\For {$t=k,k+1,\ldots,k+H$}
				\For {$i=1,2,\ldots,N$}
					\State Run {\tt KalmanPredict} by $\mathcal{UT}$
%				\State Run {\color{red} Kalman prediction} by $\mathcal{UT}:$
                    \begin{align*}
                    \left[  {\bm m}_t^{f,i}, {\bm P}_t^{\bm x, f,i}, {\bm P}_{t-1,t}^{\bm x, f,i}  \right]&=\mathcal{UT}\left( \bar{\bm f}, \bar {\bm x}_{t-1}^{f,i},  {\bm \Sigma}_{t-1}^{f,i}, \bar{\bm Q} \right) \\
                    \left[ \hat {\bar {\bm y}}_t^{i}, {\bm P}_t^{  \bm y, f,i},  {\bm P}_t^{\bm x \bm y, f,i}  \right] &=\mathcal{UT} \left( \bar {\bm h},  {\bm m}_{t}^{f,i},  {\bm P}_{t}^{\bm x, f,i}, \bar {\bm R} \right) 
					\end{align*}

				\State Run {\tt KalmanUpdate} 
%				\State Run {\color{red} Kalman update:} 
\begin{align*}
\tilde {\bm m}_t^{f, i} &=   {\bm{m}}_{t}^{f,i}  + \bm K_t^{f,i} \left( \bar{\bm y}_t - \hat{\bar{\bm y}}_t^i \right),\\ 
\tilde {\bm P}_t^{\bm x, f, i} &= {\bm P}_t^{\bm x, f, i} -  \bm K_t^{f,i}  \bm P^{  \bm y, f, i}_t \left( \bm K_t^{f,i}\right)^\top \\
 \bm K_t^{f,i}&= \bm P^{ \bm x \bm y, f, i}_t \left(  \bm P^{  \bm y, f, i}_t \right)^{-1}
\end{align*}

				\State Draw a highly probable particle $\bm \xi_t^i \sim \mathcal{N}(\bm 0, \bm I)$, and update the filtering particle 
\begin{align*}
\bar{\bm{x}}_{t}^{f, i} &=    \tilde  {\bm{m}}_{t}^{f,i} + \sqrt{\tilde {\bm P}_t^{\bm x, f, i}} \bm \xi_t^i,  \ \
\bm \Sigma_t^{f,i} = \tilde {\bm P}_t^{\bm x, f, i} \\
w_t^{f,i} &= \frac{w_{t-1}^{f,i} \p(\bar {\bm y}_t | \bar{\bm{x}}_{t-1}^{f,i})}{\sum_{j=1}^N w_{t-1}^{f,j} \p(\bar {\bm y}_t | \bar{\bm{x}}_{t-1}^{f,j})},  \ \ \bar {\bm y}_t \, | \, \bar{\bm{x}}_{t-1}^{f, i} \sim \mathcal{N} \left(\hat {\bar{ \bm y}}_t^i, \bm P^{  \bm y, f, i}_t   \right) 
\end{align*}
				\EndFor 

\State Do resampling if necessary

			\EndFor
 
\vspace{5pt}

\item[]{\hspace{13pt} \em  \color{gray}//  Backward smoothing by Kalman-IPS}

 \State Initialize the smoothing particles by $\bar{\bm x}_{k+H}^{s,i} =  \bar{\bm x}_{k+H}^{f,i}$ and $\bm \Sigma_{k+H}^{s,i} = \bm \Sigma_{k+H}^{f,i}$  for $i=1,\ldots, N$ 
			\For {$t=k+H-1,k+H-2,\ldots,k$}
				
					\For {$i=1,2,\ldots,N$} 
   
                    \State Run {\tt RTSKalmanUpdate}
%                        \State Run {\color{red} Kalman smoothing update:}
						\begin{align*} 
\tilde {\bm m}_t^{s, i} &=   \bar{\bm{x}}_{t}^{f,i}  +  \bm K_t^{s,i}  \left( \bar{\bm{x}}_{t+1}^{s,i} -  {\bm{m}}_{t+1}^{f,i} \right) \\
\tilde {\bm P}_t^{\bm x, s, i} &= {\bm \Sigma}_t^{f, i} + \bm K_t^{s,i} \left(  \bm \Sigma_{t+1}^{s,i} -  \bm P_{t+1}^{\bm x, f,i}  \right) \left(\bm K_t^{s,i}\right)^\top\\
\bm K_t^{s,i} &= \bm P^{ \bm x, f, i}_{t,t+1}   \left(  \bm P^{ \bm x,  f, i}_{t+1}  \right)^{-1} 
\end{align*}

\State Draw a highly probable particle $\bm \xi_t^i \sim \mathcal{N}(\bm 0, \bm I)$, and update the smoothing particle 
		\begin{align*}
\bar{\bm{x}}_{t}^{s, i} &=    \tilde  {\bm{m}}_{t}^{s,i} + \sqrt{\tilde {\bm P}_t^{\bm x, s, i}} \bm \xi_t^i,  \ \
\bm \Sigma_t^{s,i}  = \tilde {\bm P}_t^{\bm x, s, i}, \ \
w_t^{s,i}  = {1 \over N}
\end{align*}

					\EndFor
			\EndFor

\vspace{5pt}

\State Compute the final estimate for $\bar {\bm x}_{k}$
\begin{align*}
\hat {\bar{\bm x}}_{k}^* = {1 \over N} \sum_{i=1}^N \bar{\bm{x}}_{k}^{s, i} 
\end{align*}
\State Export and apply control  decisions

		\EndFor
  
	\end{algorithmic} 
\end{algorithm}
\end{minipage}
\end{figure*}

\subsection{MPIC via Unscented Kalman-IPF and Kalman-IPS}

We have shown that implicit particle filtering/smoothing can be approximately realized as banks of nonlinear Kalman filters/smoothers. Their exact execution will depend on which Kalman filter to use. The literature provides a range of options, which, among others, include the extended, unscented, ensemble, and cubature Kalman filters. While all these filters show utility for different applications, we note that the unscented Kalman filter is particularly suitable to enable MPIC for motion planning. First, the unscented Kalman filter offers second-order accuracy in estimation, which compares with the first-order accuracy of the extended Kalman filter. Second, it requires only modest computation for nonlinear systems with low- to medium-dimensional state spaces, and the considered NSS vehicle model falls into this case. 
Finally, its derivative-free computation eliminates a need for model linearization, which would have been burdensome for the NSS model.

Lying at the center of the unscented Kalman filter is the so-called unscented transform or $\mathcal{UT}$, which tracks the statistics in nonlinear transformations of Gaussian random vectors. Briefly, consider $\bm y = \gamma (\bm x) + \bm w$, where $\bm x$ and $\bm y$ are random vectors, $\bm w$ is a random noise vector, and $\gamma(\cdot)$ is an arbitrary nonlinear function. If we suppose  $\bm x \sim \mathcal{N} \left( \bm m^{\bm x}, \bm P^{\bm x} \right)$ and $\bm w \sim \mathcal{N} \left(\bm 0, \bm Q \right)$,   $\mathcal{UT}$ will generate the statistics  and  form a Gaussian approximation for $\bm y$:
\begin{align*}
\left[    \bm m^{\bm y}, \bm P^{\bm y}  , \bm P^{\bm x\bm y}  
\right] &= \mathcal{UT}  \left(\gamma, \bm m^{\bm x}, \bm P^{\bm x}, \bm Q \right),\\
\bm y &\sim \mathcal{N}\left( \bm m^{\bm y}, \bm P^{\bm y} \right).
\end{align*}
What $\mathcal{UT}$  does behind this is identifying a set of sigma points (deterministically chosen particles) to approximately represent $p(\bm x)$, projecting them through $\gamma(\cdot)$, and then using the transformed sigma points to compute $ \bm m^{\bm y}$, $\bm P^{\bm y} $ and $\bm P^{\bm x \bm y} $.   A more detailed description is offered in Algorithm~\ref{alg:UncentedTransform}. 

Here, we can apply $\mathcal{UT}$ to the Kalman-IPF algorithm to perform {\tt KalmanPredict} and {\tt KalmanUpdate}. Subsequent to  the forward filtering pass, the Kalman-IPS algorithm can be run. Going forward with the idea, we can use the two algorithms to execute the MPIC framework for motion planning. The resultant MPIC implementation, called MPIC-X\footnote{We refer to the algorithm   as MPIC-X to distinguish it
from the MPIC framework  that it computationally implements.}, is summarized in Algorithm~\ref{alg:MPIC-X}.

\section{Discussion}
\label{Sec:Discussion}
This section   discusses the merits, implementation aspects, and potential extensions for the above study. % MPIC-X algorithm and the MPIC framework.  

%The MPIC-X algorithm, along with its predecssors in our prior studies~\cite{}, mark the first approaches custom-built to enable MPC of highly nonlinear nonconvex NSS models, to the best of knowledge. % though it is also applicable to other types of nonlinear MPC problems. 
While growing out of the notion of MPC,  the MPIC-X algorithm breaks away from mainstream MPC realizations based on gradient optimization. The primary difference is that it builds upon the perspective of Bayesian estimation and harnesses the power of Monte Carlo sampling to infer the best control decisions. The MPIC-X algorithm will yield high computational efficiency for control of NSS models for several reasons. First, it is derivative-free and   obviates the need for generating and evaluating gradients. Second, the algorithm exploits the Markovianity of the system to perform recursive estimation for filtering and smoothing. The sequential computation as a result deals with small-sized problems one after another, improving the efficiency by large margins. Finally, the Kalman-IPF and Kalman-IPS algorithms, which together serve as the computational engine of the MPIC-X algorithm, are capable of identifying much fewer but better-quality particles to accelerate search in the state and control space.  

The MPIC-X algorithm enforces constraint adherence via multiple built-in mechanisms. First, the virtual measurement $y_{g,t}=0$, which applies to the particle update step in~\eqref{IPF-Kalman-Update}, secures generating particles within the constraints in the filtering pass. Second, $y_{g,t}=0$ also influences the importance weight evaluation in~\eqref{IPF-particle-weight}, assigning lower weights to particles that break the constraints. These particles are subsequently eliminated in the resampling step. Finally, the smoothing pass further enhances the particle quality for even greater compliance with the constraints.

%We   highlight that the MPIC-X algorithm is hardly bothered by the  particle degenecy issue. This is partly because implicit particle filtering 

Some tricks will help the MPIC-X algorithm succeed.  First, we highlight that effective inferential control would require setting up a Bayesian estimation problem for~\eqref{virtual-NSS-compact} such that $\bar{\bm y}_t$ contains enough information about $\bar {\bm x}_t$. This connects to the notions of observability and detectability for nonlinear systems. As per our experience, it often helps if we have some reference values in $\bar {\bm y}_t$ for all the unknown quantities in $\bar {\bm x}_t$, or imposing meaningful limits for these quantities if there exist any. Second, we find it useful to inflate the covariances $\bar {\bm Q}$ and $\bar {\bm R}$ with the same ratio. The covariance inflation makes no change to the original MPC formulation as it multiplies the cost function in~\eqref{cost-function} with a constant but will enlarge the space in which to sample the particles. The consequent wider searches across the state and control space will improve the inferential control performance. 
Third, warm-start will expedite the success in the filtering/smoothing passes of the MPIC-X algorithm. As a simple yet effective way, one can use $\bar {\bm x}_{k+1}^{s, i}$ for $i=1,\ldots, N$ obtained at time $k$ to initialize the filtering particles for the subsequent horizon starting at time $k+1$.  
A final trick is to replace the ideal barrier function in~\eqref{Ideal-Barrier-Funcion} with a modified softplus function
\begin{align*}
    \psi(x) = \frac{1}{a} \ln \left( 1+ e^{b  x} \right),
\end{align*}
where $a, b>0$ are tunable parameters. By tuning $a$ and $b$, one can change the shape of $\psi(x)$ to adjust the constraint satisfaction level.

Implementing the MPIC-X algorithm requires moderate parameter tuning to maximize its performance, mostly to ensure wide enough sampling ranges within the state and control space.  % for the sake of searching good motion plans. 
The tuning effort hence is largely directed towards adjusting some covariance matrices used in the algorithm. First,  $\bar{\bm Q}$ and $\bar{\bm R}$ may need to be inflated, as aforementioned.  Second,  for  the horizon $[k, k+H]$,  $ {\bm \Sigma}_k^{f} = \mathrm{diag}\left( \bm 0, \bm \Sigma^{f,\bm u}_k , \bm \Sigma_k^{f,\Delta \bm u}  \right)$, the initial filtering covariance at time $k$, should be chosen such that    the initial particles  for $\bm u_k$ and $\Delta \bm u_k$ are explorative enough. % in search for good motion plans.  
Third, to draw highly probable particles for $\bm \xi_t$, we can sample from $\mathcal{N}\left(0, \bm \Sigma_{\mathrm{IIS}} \right)$, where $\bm \Sigma_{\mathrm{IIS}} = \mathrm{diag} \left( \sigma_1 \bm I, \sigma_2 \bm I, \sigma_3 \bm I\right)$  and $\sigma_i\leq 1$ for $i=1,2,3$, so that the particles are highly probable with respect to $\mathcal{N} \left(\bm 0, \bm I \right)$.  Here, $\sigma_i$ for $i=1,2,3$ correspond to the implicit importance sampling for $\bm x_k$, $\bm u_k$, and $\Delta \bm u_k$, respectively,  with $\sigma_2,\sigma_3\gg \sigma_1$ for the sake of explorative control input. Finally, $\mathcal{UT}$ involves some tunable parameters, among which $\alpha$ adjusts the spread of the sigma points. The literature often suggests $\alpha = 10^{-3}$~\cite{Sarkka:Cambridge:2013}, but a much larger $\alpha$ is more effective here to enhance exploration due to the need of dealing with the highly nonlinear NSS model.

Various extensions are available to expand the proposed study. First, if the original MPC problem has  a non-quadratic control objective $J(\bar {\bm x}_{k:k+H}) = \sum_{t=k}^{k+H} \ell (\bar {\bm x}_t)$,  where $\ell(\cdot)$ is an arbitrary nonlinear function, we can follow~\cite{Toussaint:ICML:2009} to introduce a binary random variable $O_t \in \{0, 1 \} $   in place of $\bar {\bm y_t}$ in~\eqref{virtual-NSS-compact}  such that
\begin{align*}
\p(O_t = 1 | \bar {\bm x}_t) \propto \exp \left(   - \ell(\bar {\bm x}_t)  \right).
\end{align*}
Here,  $O_t$ measures the probability that the virtual system in~\eqref{virtual-NSS-compact} behaves optimally. 
Skipping the proof, we conclude that the corresponding MPIC problem will ask for Bayesian estimation to determine $\p ( \bar {\bm x}_{k:k+H} | O_{k:k+H}=\bm 1)$. Second,  the MPIC  framework can be implemented by more filtering/smoothing methods that provide some desirable properties. For instance,  our more recent explorations have leveraged ensemble Kalman smoothing to enable MPIC of high-dimensional nonlinear or NSS systems as they have superior computational advantages for such systems~\cite{Askari:ACC:2024,Vaziri:ACC:2025}. Third, it is of our interest to analyze the convergence properties of the MPIC-X algorithm, as an understanding of these properties will further facilitate its application. Fourth, while the MPIC-X algorithm offers high computational efficiency,  some real-world problems, including motion planning for certain robotics problems, impose extremely stringent computational demands. This raises an intriguing question: whether and how the MPIC-X algorithm can be adapted to achieve anytime computation and control.
Finally, while the MPIC framework is motivated by the motion planning problem, it lends itself well to a broader spectrum of control applications, especially those that can be dealt with by MPC.

\section{Numerical Simulations}
\label{Sec:NumSim}
In this section, we conduct a simulation study that applies the MPIC-X algorithm to motion planning in autonomous highway driving. %The simulation deals with the cases of overtaking. 
In what follows, we first describe the simulation setting, then examine the overtaking scenario with a comprehensive comparison against gradient-based MPC, and finally show the emergency braking scenario. %braking and ramp merging scenarios. 

\subsection{The Simulation Setting}

We use the MATLAB Autonomous Driving Toolbox running on a workstation with a 3.5GHz Intel Core i9-10920X CPU and 128GB of RAM to set up and implement the driving scenarios.  Driving on a structured highway, the EV can access all the necessary road information and localize itself using equipped sensors.  It acquires nominal driving specifications from a higher-level decision maker that includes waypoints and desired speeds, among others. The EV is also capable of predicting the future positions of the OVs over the upcoming planning horizon. The OV's trajectories are pre-specified using the Toolbox. In the simulation, the sampling period is $\Delta t = 0.1 \, \mathrm{s}$, which is sufficient for autonomous vehicles~\cite{Eiras:TRO:2022}, though a practitioner can take a different choice based on the specific driving requirements and available computing power. We also disregard latencies here to focus solely on assessing the performance of the MPIC-X algorithm.
%{The sampling period and integrator can be chosen arbitrarily based on the available computational resources to obtain a higher-resolution motion plan. In our simulations, we utilize the Euler integration scheme with a sampling period of $\Delta t = 0.1 \, \mathrm{s}$. Additionally, the simulations do not account for time delays when the MPIC-X algorithm's computation time exceeds the sampling period. However, the motion plans generated by the MPIC-X algorithm at the previous time step can be utilized until a new motion plan is available.}
The EV is set to maintain a safe distance of $1 \, \mathrm{m}$ for simplicity. To streamline the implementation process, we conduct motion planning in Frenet coordinates for its mathematically simpler representation, yet without loss of generality. We consider different horizon lengths for $H$ and particle numbers for $N$ when implementing the MPIC-X algorithm in the overtaking scenario for the purpose of comparison. We let $H=40$ and $N=10$ in the emergency braking scenarios.

An objective in the simulation of the overtaking scenario is to evaluate the capabilities of the MPIC-X algorithm and gradient-based MPC in handling NSS models with different structures. To this end, we consider the following feedforward neural network architectures given~\eqref{RNN_X_diff}:

\begin{itemize}

\item Net-$1$: a single hidden layer with $512$ neurons;

\item Net-$2$: two hidden layers with $128$ neurons in each layer;

\item Net-$3$: four hidden layers with $64$, $128$, $128$, and $64$ neurons in each layer. 

\end{itemize}
The activation function for all the hidden layers is the $\mathrm{tanh}$ function. For convenience, we use the single-track bicycle model to generate synthetic training datasets and then train Net-1/2/3 using the Adam optimizer. Only Net-$2$ is used in the emergency braking scenario. 
\begin{figure}[t]
\centering

\includegraphics[width=0.49\textwidth,trim={15.8cm 1.4cm 15.4cm 1.5cm},clip]{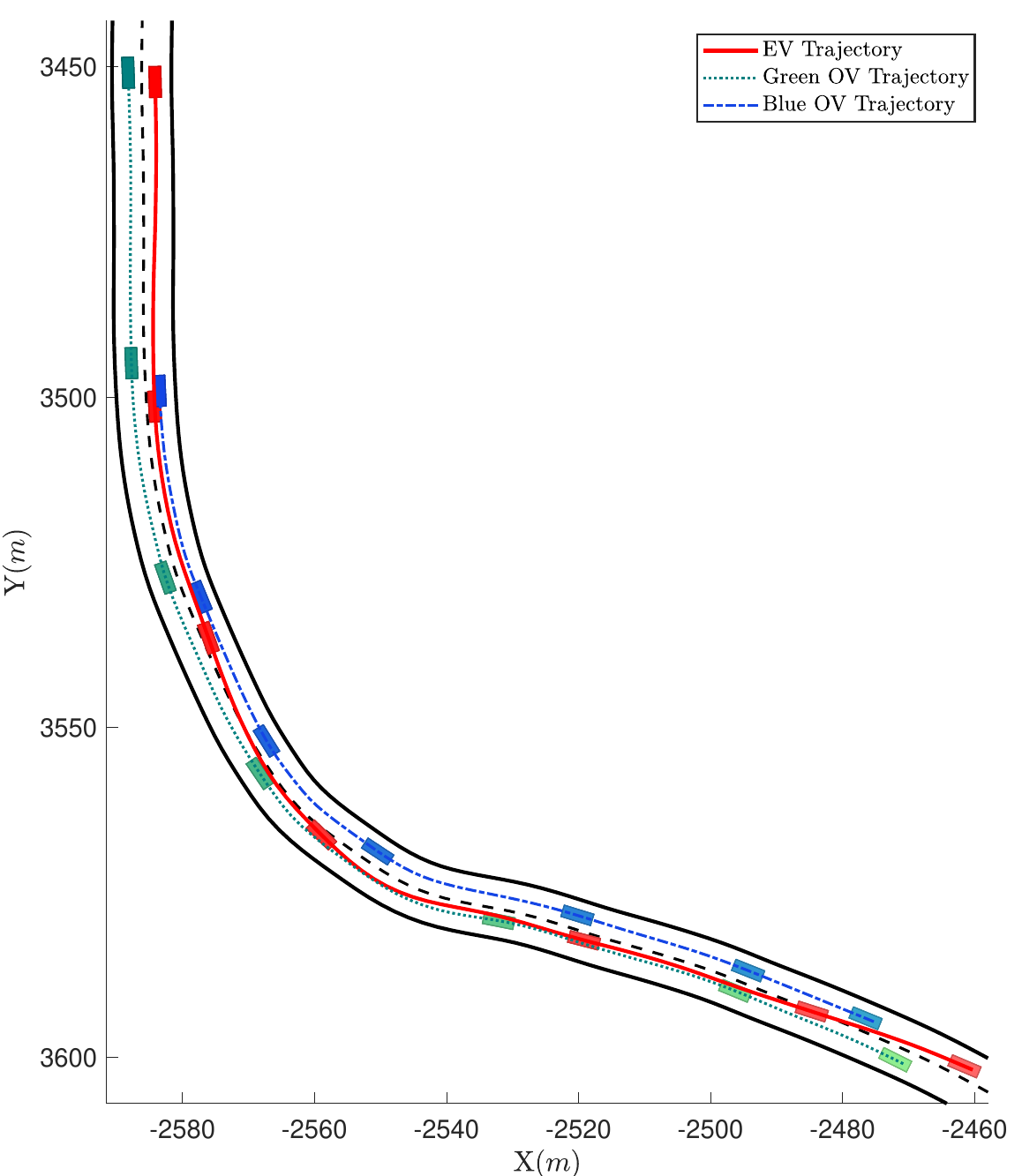}

\caption{Trajectories of the EV (in red) and OVs (in green and blue) in the overtaking scenario. The color change from light to dark indicates the passage of the driving time.}
\label{fig:OvertakingTraj}
%\vspace{-5mm}
\end{figure}
\subsection{The Overtaking Scenario}\label{SubSec:NN-dynamics}

In this scenario, the EV and OVs are running on a curved, two-lane highway road. As shown in Fig.~\ref{fig:OvertakingTraj}, the EV (in red) is initially behind two OVs (in green and blue), moving at slower speeds. To overtake the OVs, the EV computes motion plans using the MPIC-X algorithm and then maneuvers. Here, the MPIC-X algorithm adopts Net-$2$ for the NSS model and uses $N=10$ particles, with the planning horizon  $H =40$. Fig.~\ref{fig:OvertakingTraj} shows that the EV successfully overtakes the OVs without collision despite the curvature of the road. Fig.~\ref{fig:OvertakingControls} further shows the actuation profiles in  acceleration  and steering as well as the profiles of the distance between the EV and OVs. As is seen, the EV manages to comply with the driving and safety constraints throughout the maneuvering process. The results show that the MPIC-X algorithm effectively identifies motion plans to accomplish the overtaking task.

\iffalse
\begin{figure*}[t]
\begin{mdframed}[backgroundcolor=blue!10,linecolor=blue!10]
	    \centering

    \subfloat[\centering ]{{\includegraphics[trim={4.3cm 8cm 3.4cm 9.1cm},clip,width=0.3\textwidth]{Figures/Overtaking_Acc.pdf} }\label{fig:overtaking-a}} \hspace{5mm}
    \subfloat[\centering ]{{\includegraphics[trim={4.05cm 8cm 3.6cm 9.1cm},clip,width=0.305\textwidth]{Figures/Overtaking_Steer.pdf} }\label{fig:overtaking-b}}
    \hspace{5mm}
    \subfloat[\centering ]{{\includegraphics[trim={4cm 8cm 4.3cm 9.05cm},clip,width=0.3\textwidth]{Figures/Overtaking_Dist.pdf} }\label{fig:overtaking-c}}
    \caption{Simulation results for the overtaking scenario: (a) the acceleration and incremental acceleration profiles in solid curves, with respective bounds in dashed lines;  (b) the steering and incremental steering profiles in solid lines, with their respective bounds in dashed lines. (c) distance between the EV and OVs, with the dashed-line safe margin.}
\end{mdframed}
    \label{fig:OvertakingControls}
\end{figure*}
\fi

\begin{figure*}[t]
    \centering
    \begin{subfigure}[t]{0.325\textwidth}
    \includegraphics[width=\textwidth,trim={4.3cm 8cm 3.4cm 9.1cm},clip]{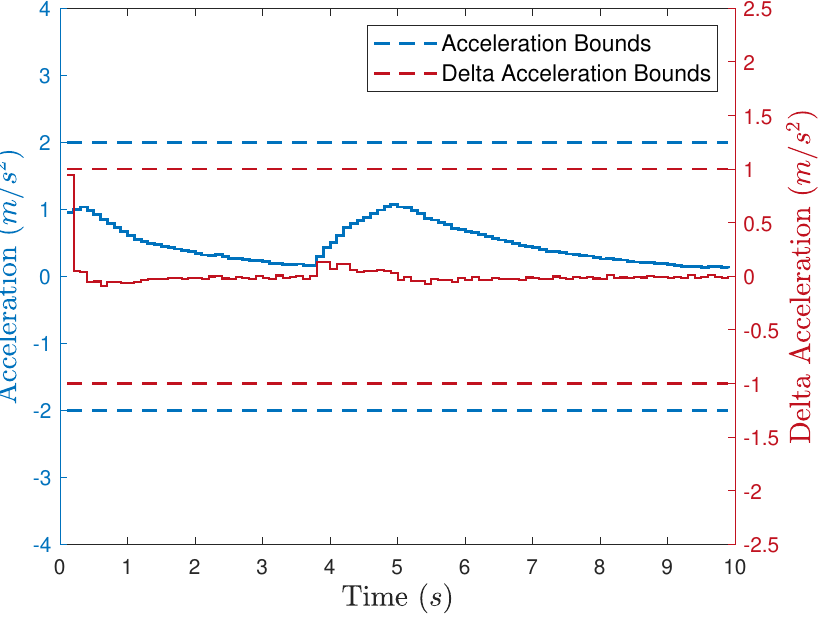}
    \caption{}
    \label{fig:overtaking-a}
\end{subfigure}
\hspace{1mm}
\begin{subfigure}[t]{0.325\textwidth}
    \includegraphics[width=\textwidth,trim={4.05cm 8cm 3.6cm 9.1cm},clip]{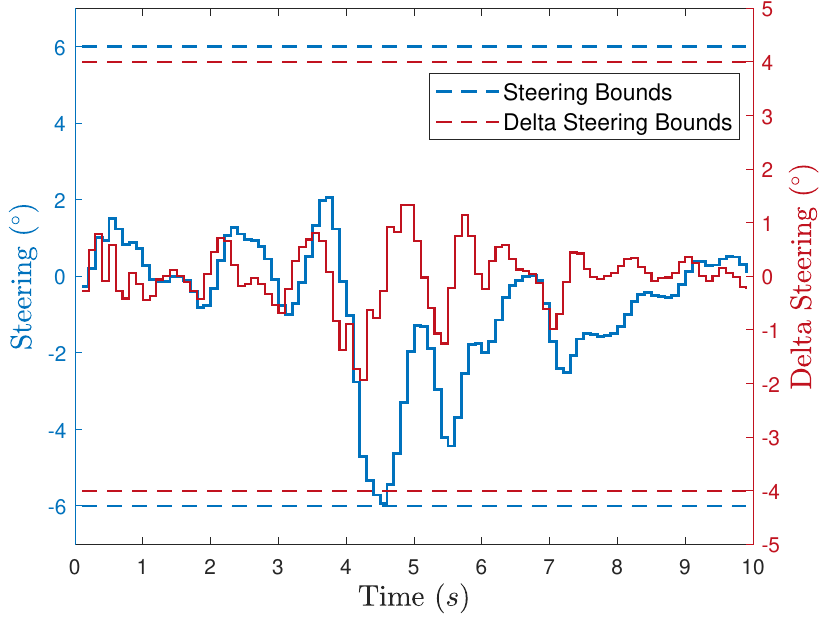}
    \caption{}
    \label{fig:overtaking-b}
\end{subfigure}
\hspace{1mm}
\begin{subfigure}[t]{0.31\textwidth}
    \includegraphics[width=\textwidth,trim={4cm 8cm 4.3cm 9.05cm},clip]{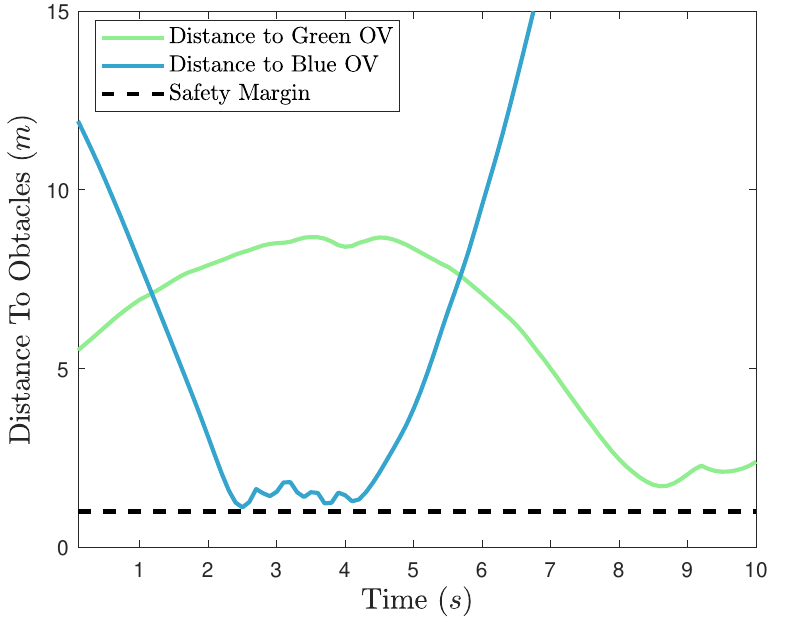}
    \caption{}
    \label{fig:overtaking-c}
\end{subfigure}
    \caption{Simulation results for the overtaking scenario: (a) the acceleration and incremental acceleration profiles in solid curves, with respective bounds in dashed lines; (b) the steering and incremental steering profiles in solid lines, with their respective bounds in dashed lines; (c) distance between the EV and OVs, with the dashed-line safe margin.}
    \label{fig:OvertakingControls}
\end{figure*}

\begin{figure*}[h!]
    \centering
    \begin{subfigure}[t]{0.325\textwidth}
    \includegraphics[width=\textwidth,trim={3.9cm 8cm 3.4cm 8.6cm},clip]{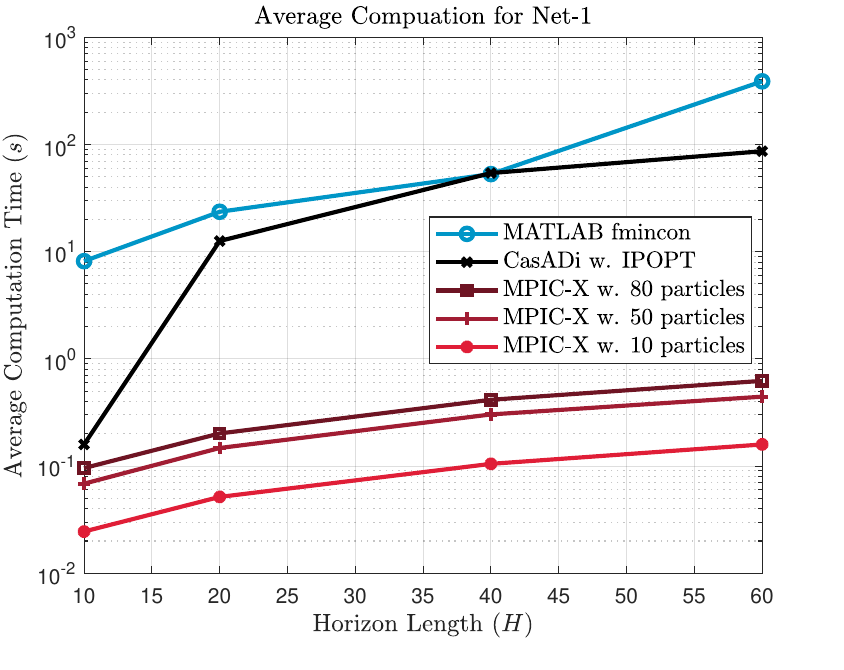}
    \caption{}
    \label{fig:computation-a}
\end{subfigure}
\hspace{1mm}
\begin{subfigure}[t]{0.32\textwidth}
    \includegraphics[width=\textwidth,trim={3.9cm 8cm 3.6cm 8.6cm},clip]{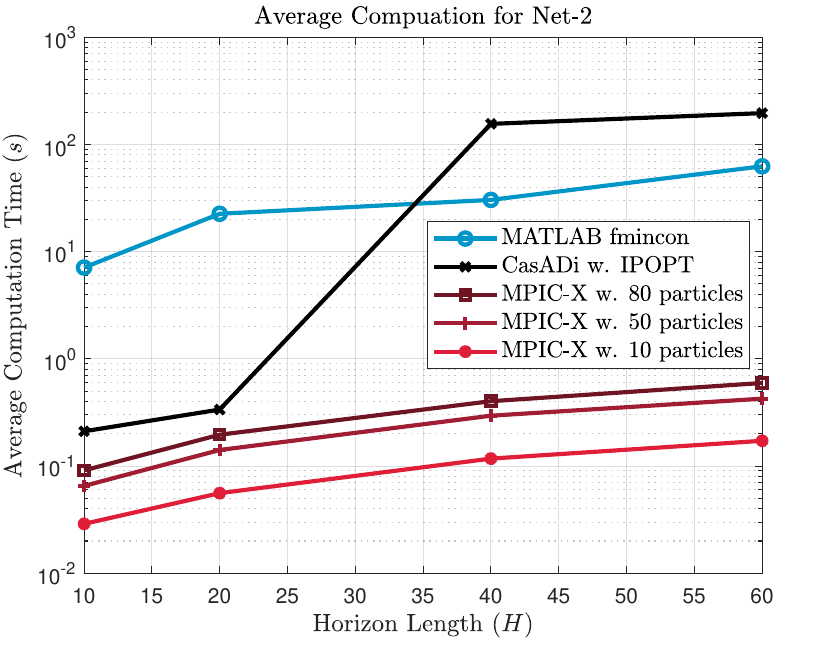}
    \caption{}
    \label{fig:computation-b}
\end{subfigure}
\hspace{1mm}
\begin{subfigure}[t]{0.32\textwidth}
    \includegraphics[width=\textwidth,trim={3.9cm 8cm 3.6cm 8.6cm},clip]{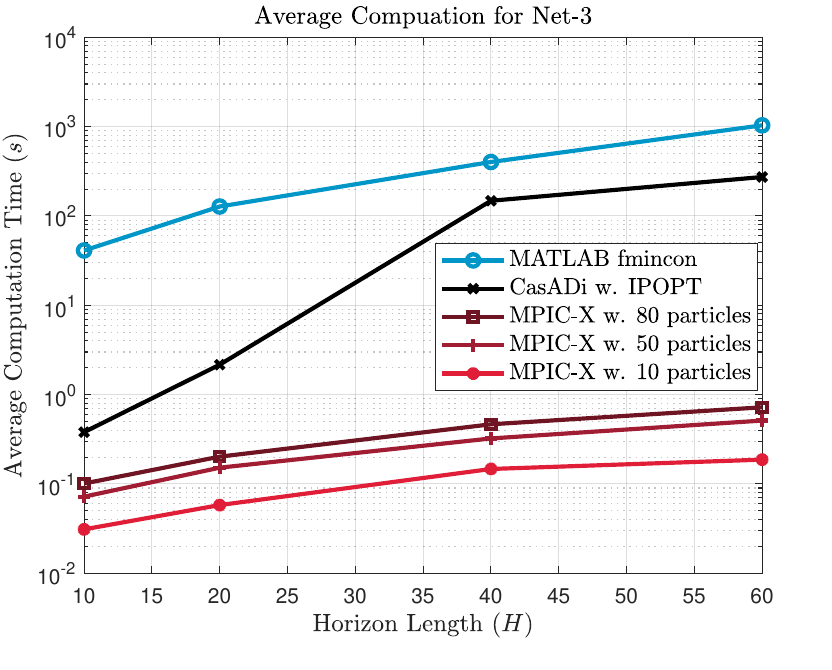}
    \caption{}
    \label{fig:computation-c}
\end{subfigure}
    \caption{Computational time comparison between MPIC-X algorithm and gradient-based MPC for different networks, horizon lengths, and particle numbers: (a) Net-1; (b) Net-2; (c) Net-3.}
    \label{fig:Computation-Nets}
\end{figure*}

\begin{table*}[!htbp] 
\caption{Numerical comparison of MPIC-X and gradient-based MPC}

\resizebox{1\textwidth}{!}{
 \begin{tabular}{c  c  c  c  c  c  c}
\toprule
Network & \makecell[c]{Horizon \\ ($H$)} & Method & Total Cost & \makecell[c]{ Average \\ Computation \\ Time (s)} & \makecell[c]{Relative Cost \\ Change w.r.t \\ IPOPT \& {\tt fmincon} (\%)} & \makecell[c]{Relative Computation \\ Time Change w.r.t \\ IPOPT \& {\tt fmincon} (\%)}  \\
\midrule
\multirow{19.5}{2cm}{\makecell[c]{Net-$1$} \\ $\qquad \; \, {\scriptstyle (512)}$}
& \multirow{5.5}{2cm}{\makecell[c]{$10$}} & CasADi w. IPOPT & $16,808$ & $0.159$ & --- &  --- \\
& & MATLAB {\tt fmincon} & $17,922$ & $8.162$ & --- & --- \\
 & & MPIC-X w. $10$ particles & $18,574$ & $0.025$ &  $10.51$\, \& \,$3.64$ & $-84.28$\, \& \,$-99.69$ \\
 & & MPIC-X w. $50$ particles & $18,158$ & $0.069$ &  $8.04 $\, \, \& \,$1.32$ & $-56.60$\, \& \,$-99.15$ \\
 & & MPIC-X w. $80$ particles & $18,077$ & $0.096$ &  $7.56 $\, \, \& \,$0.86$ & $-39.62$\, \& \,$-98.82$ \\
 \cmidrule(l){2-7}
 & \multirow{5.5}{2cm}{\makecell[c]{$20$}} & CasADi w. IPOPT & $63,197$ & $12.58$ & --- & --- \\
& & MATLAB {\tt fmincon} & $18,862$ & $23.57$ &  --- & --- \\
 & & MPIC-X w. $10$ particles & $18,235$ & $0.052$ &  $-71.15$\, \& \,$-3.33$ & $-99.59$\, \& \,$-99.78$ \\
 & & MPIC-X w. $50$ particles & $17,701$ & $0.147$ & $-71.99$\, \& \,$-6.16$ & $-98.83$\, \& \,$-99.37$ \\
 & & MPIC-X w. $80$ particles & $17,602$ & $0.202$ &  $-72.15$\, \& \,$-6.68$ & $-98.40$\, \& \,$-99.15$ \\
 \cmidrule(l){2-7}
  & \multirow{5.5}{2cm}{\makecell[c]{$40$}} & CasADi w. IPOPT & --- & --- & --- & ---  \\
& & MATLAB {\tt fmincon} & $30,103$ & $52.96$ & --- & --- \\
 & & MPIC-X w. $10$ particles & $18,596$ & $0.105$ & ---\, \& \,$-38.22$ & ---\, \& \,$-99.80$ \\
 & & MPIC-X w. $50$ particles & $17,944$ & $0.303$ & ---\, \& \,$-40.39$ & ---\, \& \,$-99.43$ \\
 & & MPIC-X w. $80$ particles & $17,890$ & $0.415$ & ---\, \& \,$-40.57$ & ---\, \& \,$-99.22$ \\
 \cmidrule(l){2-7}
 & \multirow{5.5}{2cm}{\makecell[c]{$60$}}& CasADi w. IPOPT & --- & --- & --- &  --- \\
& & MATLAB {\tt fmincon} & --- & --- & --- & --- \\
 & & MPIC-X w. $10$ particles & $18,496$ & $0.159$ & --- & --- \\
 & & MPIC-X w. $50$ particles & $18,053$ & $0.444$ & --- & --- \\
 & & MPIC-X w. $80$ particles & $17,940$ & $0.624$ & --- & --- \\
\midrule
\multirow{19.5}{2cm}{\makecell[c]{Net-$2$}\\ $\quad \; \,{\scriptstyle (128-128)}$}
& \multirow{5.5}{2cm}{\makecell[c]{$10$}} & CasADi w. IPOPT & $16,650$ & $0.211$ & --- &  --- \\
& & MATLAB {\tt fmincon} & $18,833$ & $7.111$ & --- & --- \\
 & & MPIC-X w. $10$ particles & $18,992$ & $0.029$ & $17.31$ \& $\phantom{-}3.71$ & $-86.26$\, \& \,$-99.59$ \\
 & & MPIC-X w. $50$ particles & $18,923$ & $0.065$ & $13.65$ \& $-0.48$ & $-69.10$\, \& \,$-99.12$ \\
 & & MPIC-X w. $80$ particles & $18,778$ & $0.091$ & $12.77$ \& $-0.29$ & $-56.82$\, \& \,$-98.72$ \\
 \cmidrule(l){2-7}
 & \multirow{5.5}{2cm}{\makecell[c]{$20$}} & CasADi w. IPOPT & $19,499$ & $0.337$ & --- & --- \\
& & MATLAB {\tt fmincon} & $17,209$ & $22.58$ &  --- & --- \\
 & & MPIC-X w. $10$ particles & $18,064$ & $0.055$ & $-7.36$ \& $4.97$ & $-83.40$\, \& \,$-99.75$ \\
 & & MPIC-X w. $50$ particles & $17,633$ & $0.141$ & $-9.57$ \& $2.46$ & $-58.04$\, \& \,$-99.37$ \\
 & & MPIC-X w. $80$ particles & $17,473$ & $0.196$ & $-10.4$ \& $1.53$ & $-41.74$\, \& \,$-99.13$ \\
 \cmidrule(l){2-7}
  & \multirow{5.5}{2cm}{\makecell[c]{$40$}} & CasADi w. IPOPT & --- & --- & --- & ---  \\
& & MATLAB {\tt fmincon} & --- & --- & --- & --- \\
 & & MPIC-X w. $10$ particles & $18,392$ & $0.117$ & --- & --- \\
 & & MPIC-X w. $50$ particles & $17,833$ & $0.295$ & --- & --- \\
 & & MPIC-X w. $80$ particles & $17,692$ & $0.403$ & --- & --- \\
 \cmidrule(l){2-7}
 & \multirow{5.5}{2cm}{\makecell[c]{$60$}}& CasADi w. IPOPT & --- & --- & --- &  --- \\
& & MATLAB {\tt fmincon} & --- & --- & --- & --- \\
 & & MPIC-X w. $10$ particles & $18,535$ & $0.172$ & --- & --- \\
 & & MPIC-X w. $50$ particles & $17,932$ & $0.426$ & --- & --- \\
 & & MPIC-X w. $80$ particles & $17,746$ & $0.596$ & --- & --- \\
\midrule
\multirow{19.5}{2cm}{\makecell[c]{Net-$3$}\\ ${\scriptstyle (64-128-128-64)}$}
& \multirow{5.5}{2cm}{\makecell[c]{$10$}} & CasADi w. IPOPT & $16,396$ & $0.380$ & --- &  --- \\
& & MATLAB {\tt fmincon} & $17,086$ & $41.04$ & --- & --- \\
 & & MPIC-X w. $10$ particles & $18,551$ & $0.031$ &  $13.14$\, \& \,$8.57$ & $-91.84$\, \& \,$-99.92$ \\
 & & MPIC-X w. $50$ particles & $18,047$ & $0.072$ &  $10.07$\, \& \,$5.62$ & $-81.05$\, \& \,$-99.82$ \\
 & & MPIC-X w. $80$ particles & $18,017$ & $0.101$ &  $9.89\phantom{0}$\, \& \,$5.45$ & $-73.4\phantom{0}$\, \& \,$-99.75$ \\
 \cmidrule(l){2-7}
 & \multirow{5.5}{2cm}{\makecell[c]{$20$}} & CasADi w. IPOPT & $26,088$ & $2.163$ & --- & --- \\
& & MATLAB {\tt fmincon} & $15,549$ & $127.5$ &  --- & --- \\
 & & MPIC-X w. $10$ particles & $17,434$ & $0.058$ & $-33.17$\, \& \,$12.13$ & $-99.59$\, \& \,$-99.78$ \\
 & & MPIC-X w. $50$ particles & $16,970$ & $0.152$ & $-34.95$\, \& \,$9.14\phantom{0}$ & $-98.83$\, \& \,$-99.37$ \\
 & & MPIC-X w. $80$ particles & $16,853$ & $0.202$ & $-35.40$\, \& \,$8.39\phantom{0}$ & $-98.40$\, \& \,$-99.15$ \\
 \cmidrule(l){2-7}
  & \multirow{5.5}{2cm}{\makecell[c]{$40$}} & CasADi w. IPOPT & --- & --- & --- & ---  \\
& & MATLAB {\tt fmincon} & --- & --- & --- & --- \\
 & & MPIC-X w. $10$ particles & $17,961$ & $0.147$ & --- & --- \\
 & & MPIC-X w. $50$ particles & $17,291$ & $0.323$ & --- & --- \\
 & & MPIC-X w. $80$ particles & $17,363$ & $0.465$ & ---- & ---- \\
 \cmidrule(l){2-7}
 & \multirow{5.5}{2cm}{\makecell[c]{$60$}}& CasADi w. IPOPT & --- & --- & --- &  --- \\
& & MATLAB {\tt fmincon} & --- & --- & --- & --- \\
 & & MPIC-X w. $10$ particles & $18,135$ & $0.187$ & --- & --- \\
 & & MPIC-X w. $50$ particles & $17,600$ & $0.512$ & --- & --- \\
 & & MPIC-X w. $80$ particles & $17,361$ & $0.721$ & --- & --- \\
\bottomrule
\end{tabular}
 }
 
\label{Table: Computation-Comp}
\noindent{\footnotesize 
\textsuperscript{*}
``---'' indicates a failure to converge and find optima within the specified maximum of 5,000 iterations.}
\end{table*}

Next, we compare the MPIC-X algorithm against gradient-based MPC in terms of computation time and planning performance. 
 To solve gradient-based MPC, we use two solvers: MATLAB's {\tt fmincon} with the default interior point method, and CasADi~\cite{Andersson:MPC:2018} with IPOPT~\cite{Wchter:MP:2005}.

%{\color{blue} We utilize MATLAB's {\tt fmincon} with its interior point method solver and the state-of-the-art CasADi framework with IPOPT solver as benchmark gradient-based MPC solvers.} % to solve the constrained optimization problem. 
For a fair comparison, the closed form of each gradient is pre-determined offline and then called during the online optimization process, and the optimization is also warm-started. The stopping tolerances for {\tt fmincon} and CasADi-IPOPT are set close to each other to ensure impartiality in performance evaluation.  To achieve a comprehensive assessment, we perform the simulation for different planning horizons by letting  $ H = 10$, $20$, $40$, and $60$,  and for the MPIC-X algorithm, we use different particle numbers,   $N = 10$, $50$,  and $80$. Each setting comes with ten simulation runs to find out the average computation time.   

Table~\ref{Table: Computation-Comp} summarizes the simulation results and comparisons. The overarching observations are as follows.

\begin{figure*}[h]
    \centering
        \centering
    \begin{subfigure}[t]{0.25\textwidth}
        \includegraphics[trim={3.4cm 8cm 4.0cm 8.5cm},clip,width=\textwidth]{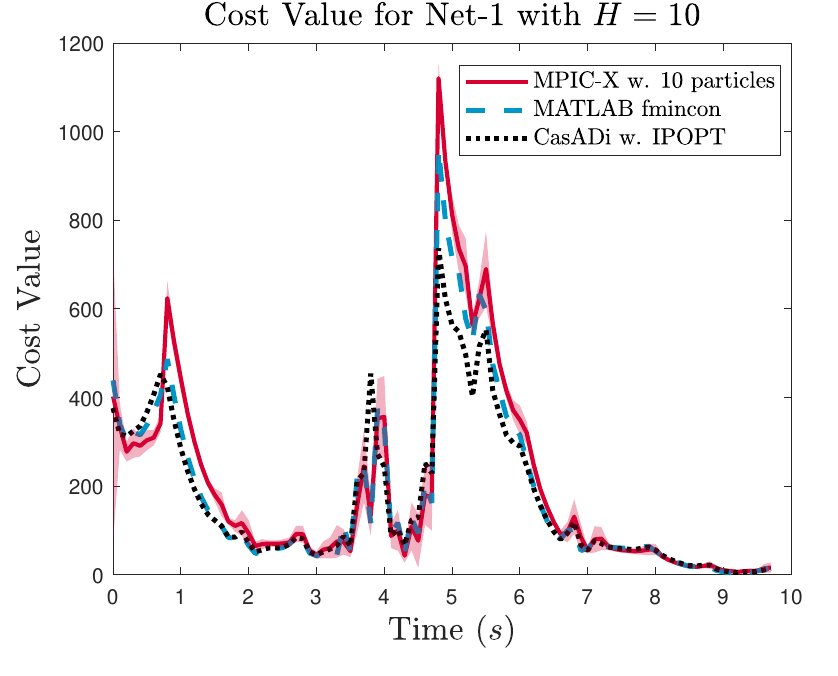}
        \caption{}
        \label{fig:Cost_Net2_H10}
    \end{subfigure}
    \hspace{-4mm}
    \begin{subfigure}[t]{0.25\textwidth}
        \includegraphics[trim={3.4cm 8cm 4.0cm 8.5cm},clip,width=\textwidth]{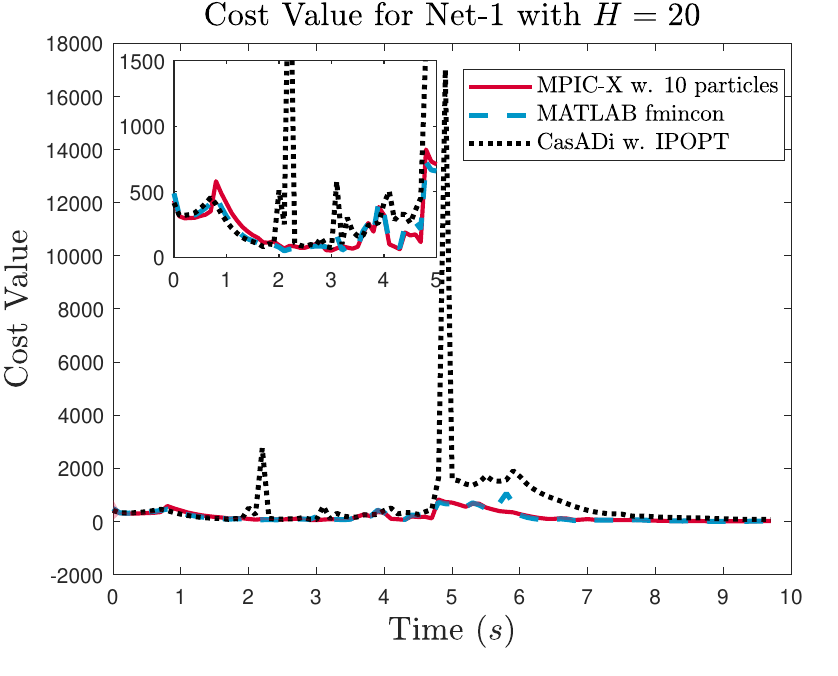}
        \caption{}
        \label{fig:Cost_Net2_H20}
    \end{subfigure}
    \hspace{-4mm}
    \begin{subfigure}[t]{0.25\textwidth}
        \includegraphics[trim={3.4cm 8cm 4.0cm 8.5cm},clip,width=\textwidth]{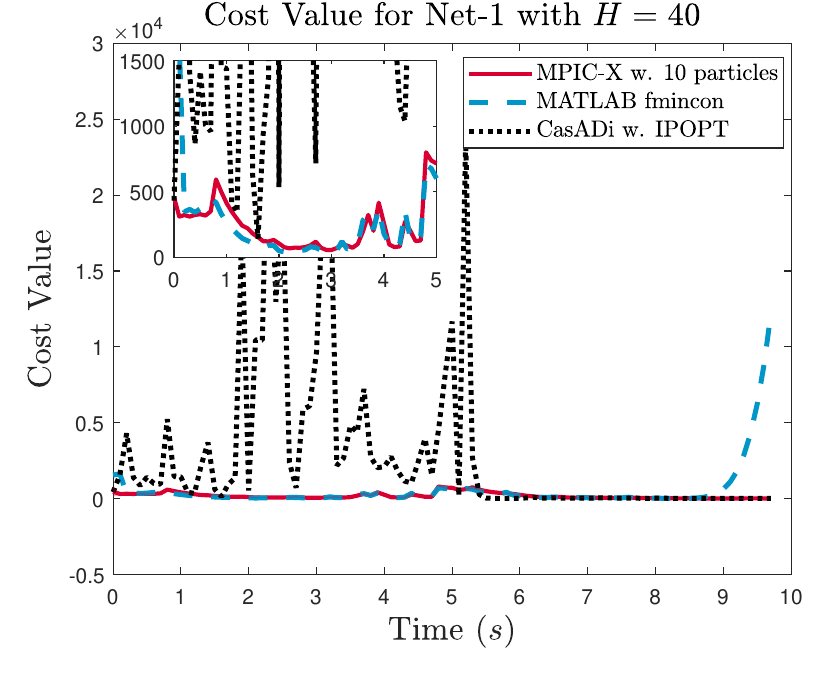}
        \caption{}
        \label{fig:Cost_Net2_H40}
    \end{subfigure}
    \hspace{-4mm}
    \begin{subfigure}[t]{0.245\textwidth}
        \includegraphics[trim={3.8cm 8cm 4.0cm 8.5cm},clip,width=\textwidth]{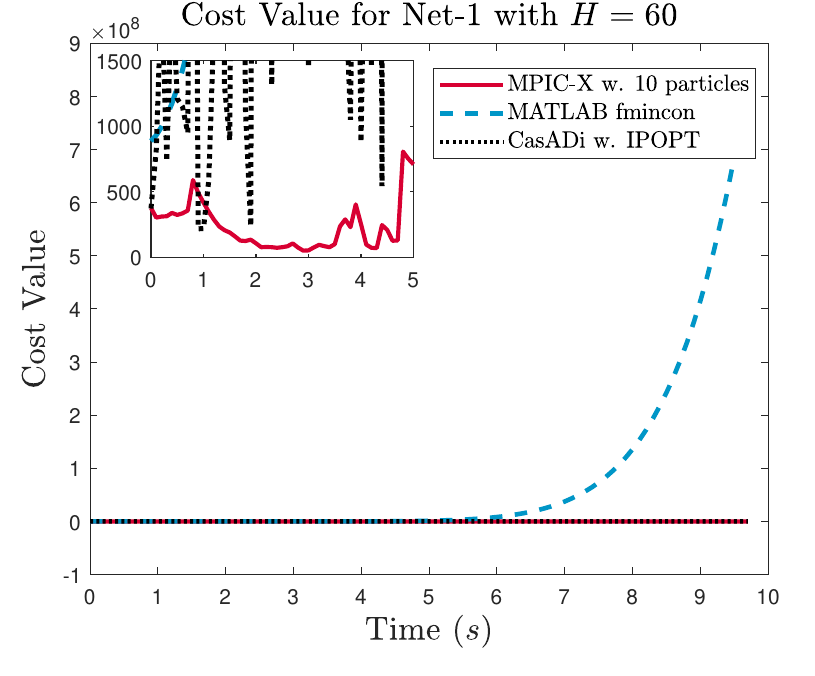}
        \caption{}
        \label{fig:Cost_Net2_H60}
    \end{subfigure}
    \caption{Cost performance comparison between the MPIC-X algorithm with ten particles and the gradient-based MPC for Net-1 at different horizon lengths: (a) $H=10$;  (b) $H=20$; (c) $H=40$; (d) $H=60$. The red-shaded region represents $\pm 2\sigma$ bounds for ten Monte Carlo runs of the MPIC-X algorithm.}
    \label{fig:Cost-Net1}
\end{figure*}

\begin{itemize}
\item For the gradient-based MPC, IPOPT is faster than {\tt fmincon}  by one to two orders of magnitude, despite their comparable cost performance. For both solvers, their computation time increases as the planning horizon $H$ extends from $10$ to $20$, with the increase for {\tt fmincon} more conspicuous. However, both IPOPT and {\tt fmincon} fail to find optima when the horizon increases further, e.g., $H=40$ and $60$, within the specified maximum of 5,000 iterations. The neural network architecture, especially the number of layers, also plays a crucial role in the computation. For both IPOPT and {\tt fmincon}, the computation time is notably more for Net-3 than for Net-1/2, as the more hidden layers in Net-3 introduce stronger nonlinearity and nonconvexity. The architectural complexity appears to impact {\tt fmincon} more.  These results reflect the limitations and struggle of the gradient-based MPC in control of NSS models, especially under long planning horizons or when employing sophisticated neural network architectures.

\item The MPIC-X algorithm demonstrates strong capabilities in controlling NSS models. While    IPOPT is much faster than {\tt fmincon} as shown in Table~\ref{Table: Computation-Comp}, the MPIC-X algorithm outpaces IPOPT  by a substantial degree. Using only $N=10$ particles, it reduces computation time by at least $80\%$ compared to IPOPT. Although   the computation unsurprisingly increases as $N$ grows for the MPIC-X algorithm, the increase rates are less than linear, and the algorithm    always outperforms IPOPT and {\tt fmincon}, by at least $40\%$. It is  important to emphasize the success of the MPIC-X algorithm in addressing motion planning even when the planning horizon $H=40$ and $60$,  where IPOPT and {\tt fmincon} fail  to converge. Another appealing feature of the MPIC-X algorithm is its computational insensitivity to the neural network architecture. Due to its sampling-based nature, the algorithm achieves nearly indistinguishable increase in computation time for Net-1/2/3, making it highly advantageous for controlling NSS models with varying levels of neural network complexity. 
Further, the MPIC-X algorithm generates slightly higher, yet still close, costs than IPOPT and {\tt fmincon} in the cases when the latter  manage to converge. The simulations constantly show that the MPIC-X algorithm with just $N=10$ particles can deliver good enough cost performance and outstanding computational efficiency.  

\end{itemize}

Fig.~\ref{fig:Computation-Nets} illustrates the higher computational efficiency of the MPIC-X algorithm by orders of magnitude than the gradient-based MPC. Both point to the outstanding computational merits and scalability of the MPIC-X algorithm. Fig.~\ref{fig:Cost-Net1} demonstrates the cost performance over time. What it shows reinforces what is described above---the MPIC-X algorithm presents comparable performance when the planning horizon is $H=10$ and $20$, and performs much better when $H=40$ and $60$. 

Finally, a side note is that IPOPT produces a much higher cost and requires longer computation time for Net-1 when $H=20$, which appears inconsistent with its performance in other cases. Something similar is also seen for {\tt fmincon} for Net-1 when $H=40$. This is because  Net-1's mildly lower predictive accuracy at $H=20$ and $H=40$, due to its simpler architecture, makes the computed gradients and Hessians deviate significantly from what is true, thereby weakening the optimization search. We draw from this observation that the NSS model must be sufficiently precise to make gradient-based MPC effective.

\subsection{The Braking Scenario}\label{SubSec:Congestion}

A leading cause of  highway traffic accidents   is the build-up of traffic congestion.  When the OVs ahead come to a complete stop, the EV must be able to decelerate from a high speed and brake to zero speed,  while considering passenger comfort and avoiding collision. We simulate this scenario by applying the MPIC-X algorithm with $N=10$ particles and utilizing   Net-$2$ for the vehicle model.

\begin{figure*}[t]\centering
\includegraphics[width=1\textwidth, trim={3.5cm 11.1cm 5cm 6.5cm},clip,,width=0.98\textwidth ]{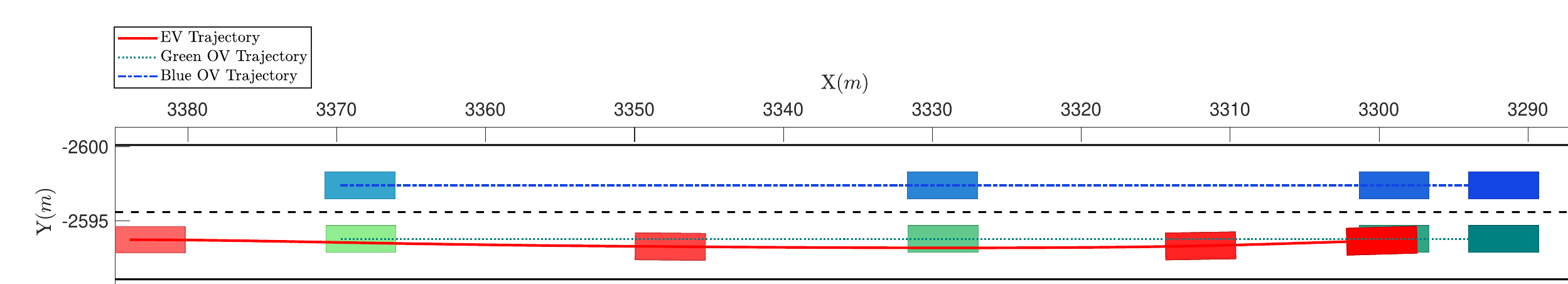}
\caption{Trajectories of the EV (in red) and OVs (in green and blue) in the braking scenario.}
\label{fig:ConjestionTraj}
%\vspace{5mm}
\end{figure*}

\begin{figure*}[!htbp]
	    \centering
    \subfloat[\centering ]{{\includegraphics[trim={4.15cm 8cm 3.6cm 9.1cm},clip,width=0.325\textwidth]{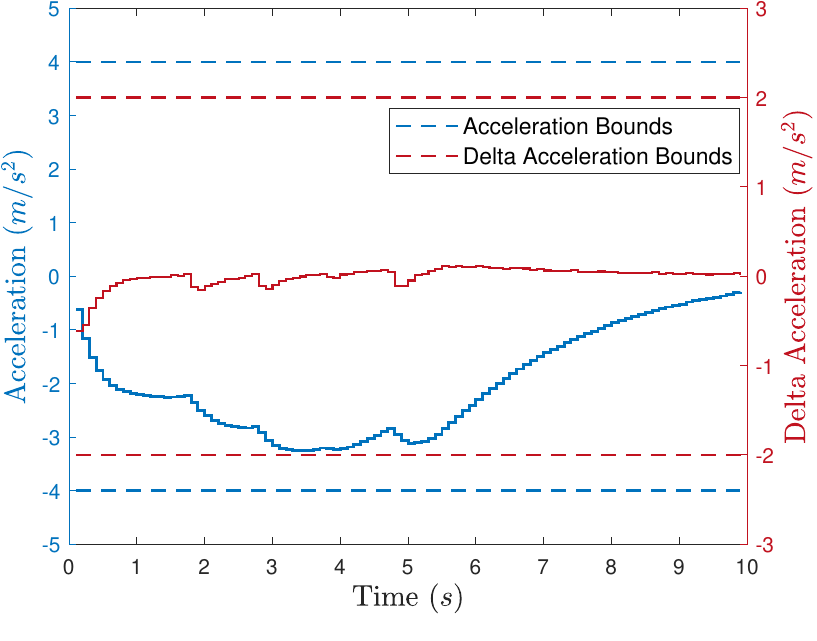} }\label{fig:cong-a}}
   \hspace{1mm}
    \subfloat[\centering ]{{\includegraphics[trim={4.15cm 8cm 4.3cm 9.1cm},clip,width=0.315\textwidth]{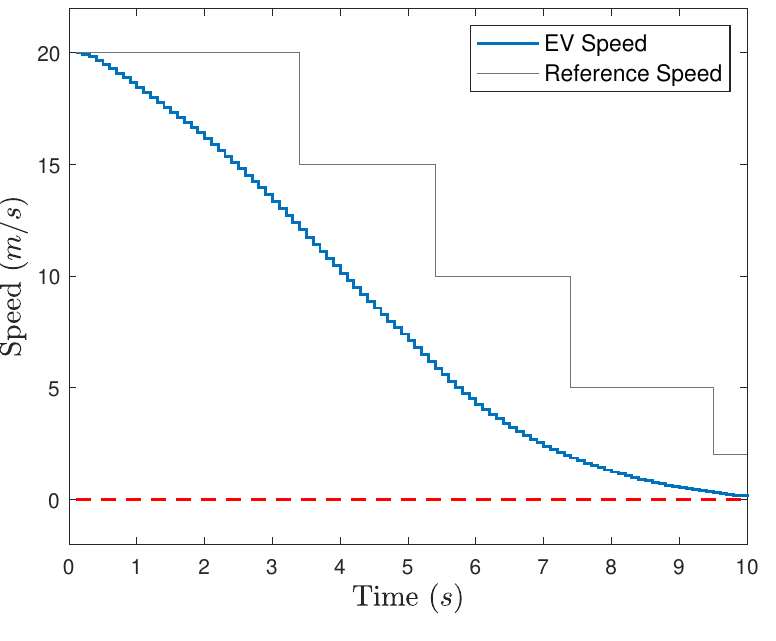} }\label{fig:cong-b}}
    \hspace{1mm}
    \subfloat[\centering ]{{\includegraphics[trim={4.15cm 8cm 4.3cm 9.1cm},clip,width=0.315\textwidth]{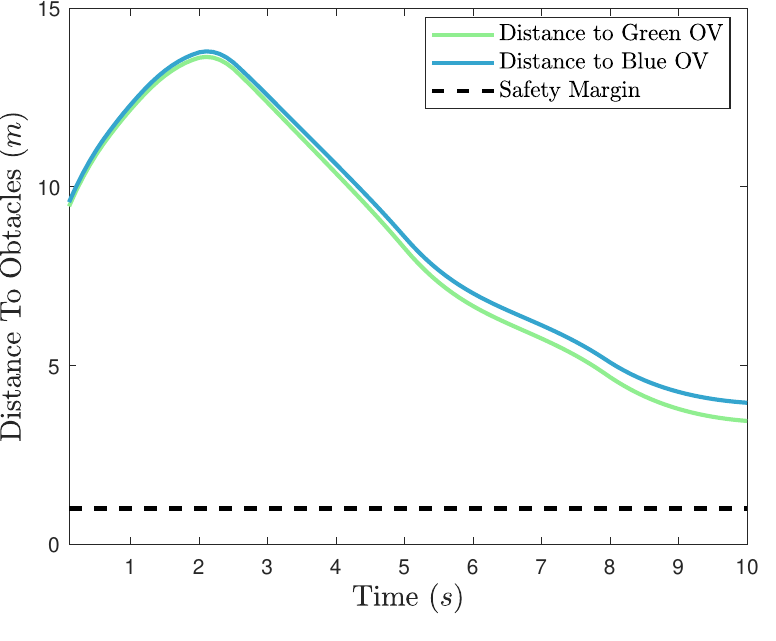} }\label{fig:cong-c}}
    \caption{Simulation results for the braking scenario: (a) the acceleration and incremental acceleration profiles in solid curves, with respective bounds in dashed lines;  (b) the EV speed profile and reference speed from the higher-level planner. (c) the distance between the EV and OVs, with the dashed-line safe margin.}
    \label{fig:CongestionControls}
\end{figure*}

Fig.~\ref{fig:ConjestionTraj} depicts the trajectories of the EV and OVs. Initially, the OVs travel faster than the EVs, and then the green OV rapidly decelerates to a complete stop. During the first three seconds, the higher-level decision maker is assumed to be unaware of the upcoming congestion, keeping the nominal speed unaltered, as seen in Fig.~\ref{fig:cong-b}. Despite this, the EV starts to decelerate. Meanwhile, the EV slightly deviates from the lane center in search of a feasible trajectory that can maintain the nominal speed, as seen in~Fig.~\ref{fig:ConjestionTraj}. However, no such trajectory exists as the OVs occupy both lanes, so the EV decides to stay in its current lane and brake to avoid a collision. This behavior demonstrates the effectiveness of MPIC's constraint awareness to ensure collision-free planning while adhering to all driving constraints, as shown in Fig.~\ref{fig:CongestionControls}..

\vspace{8pt}

The above simulation results show that the MPIC-X algorithm and the MPIC framework are effective at finding motion plans for autonomous vehicles and enabling control of NSS models. Based on the results, we highlight again that: 1) the MPIC-X algorithm offers high computational efficiency in solving MPC of NSS models, 2) its computational performance is almost insensitive to neural network architectures, 3) it can achieve good enough cost performance with just a few particles, and 4) the algorithm can succeed in solving large-scale, long-horizon MPC problems where gradient-based optimization may easily fail.

\begin{figure}[t]   \centering\includegraphics[width=0.45\textwidth,trim={1.8cm 5.2cm 2cm 6.3cm},clip]{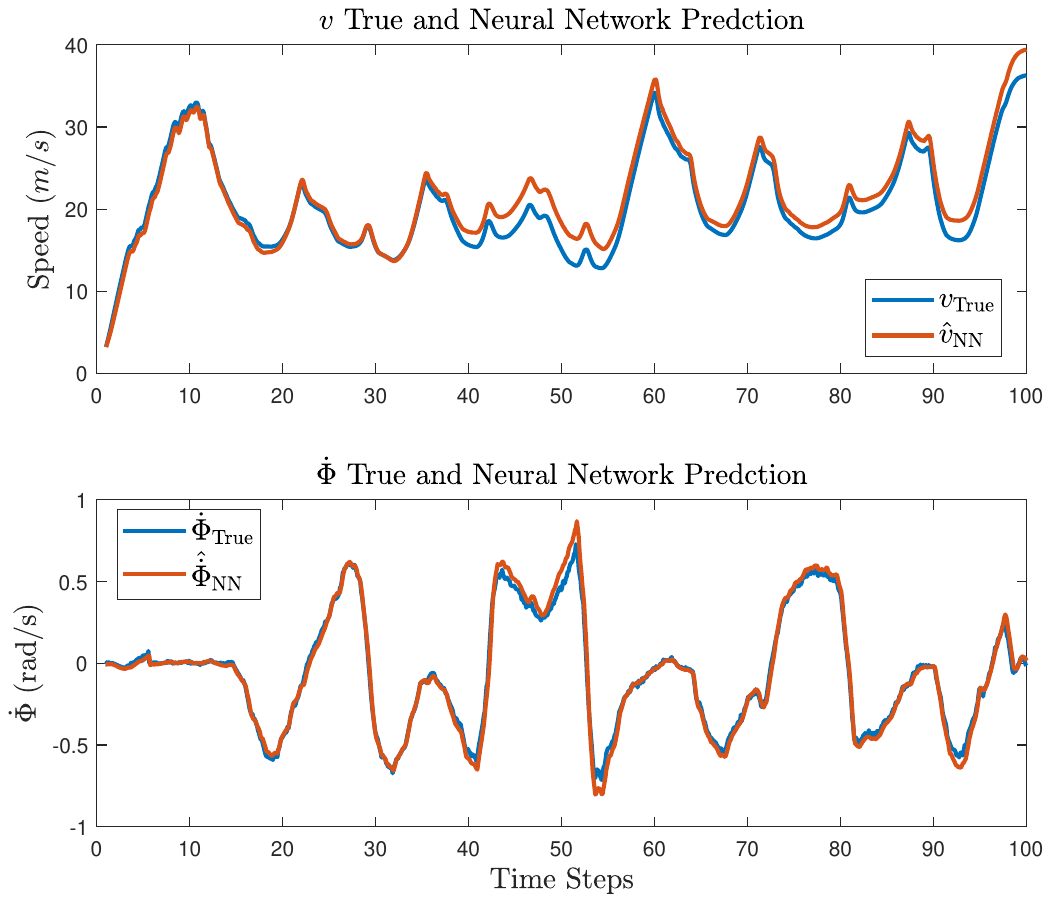}  \caption{Prediction performance of the NSS model trained for the DevBot2.0 car over a data sequence.}
    \label{fig:Exp_NN_test}
\end{figure}

\section{Motion Planning for a Real-World Vehicle}
\label{sec:Experimental}

In this section, we aim to evaluate the  potential practical performance of the MPIC-X algorithm. To achieve this, we first train an NSS model  for  a real-world autonomous vehicle  and then apply the MPIC-X algorithm to the model for motion planning.

The considered vehicle is a full-size Roborace DevBot 2.0 racing car, developed at the Technical University of Munich~\cite{Herman:CIST:2020,IAC:JFR:2023}. The car is equipped with a suite of sensors to measure the longitudinal/lateral velocity/acceleration, yaw rate, steering angle, and other parameters. Data were collected from the car driving on different race tracks at high speeds    and under varying tire-road friction  scenarios. As shown in~\cite{Herman:CIST:2020},    end-to-end neural networks are more accurate in capturing the car's dynamic behavior than the single-track model, especially in highly dynamic situations, given the availability of sufficient data. Also, neural network-based modeling can   easily accommodate  different tire-road friction levels, while physics-based modeling either fails or requires 
  manual calibration. For this car, high-fidelity  dynamics simulation data   are publicly available  at~\cite{TUM:GitHub:2021} and used here to train an NSS model for motion planning here.

The NSS model uses a residual feedforward neural network comprising three hidden layers, each with $256$ neurons using the activation function $\mathrm{tanh}$. 
We streamline the dataset by averaging every 12 time steps, thereby enhancing the efficiency of the training process. The NSS model must be trained to be accurate for a minimum of $H$ steps to allow motion planning on horizons of length $H$.  Fig.~\ref{fig:Exp_NN_test} shows  the prediction performance of the trained model. As is seen, the model can accurately predict $V$ and $\dot \Phi$ in the  longitudinal dynamics over $100$ steps, making it sufficient for the motion planning task.   % It performs consistently despite the $100$ time step predictions. Therefore, the trained residual network inhibits sufficient accuracy required to perform motion planning of up to $H = 100$.

Next, the MPIC-X algorithm is applied to the trained NSS model to perform motion planning for overtaking. Fig.~\ref{fig:Exp_ConjestionTraj} illustrates that the red EV successfully overtakes the blue and green OVs with two lane changes. Figs.~\ref{fig:Exp_overtaking-a}-\ref{fig:Exp_overtaking-b} present the actuation profiles, while Fig.~\ref{fig:Exp_overtaking-c} depicts the distances between the EV and the OVs. Initially driving behind the green OV, the EV begins with mild acceleration and executes a lane change. Once a safe distance is achieved relative to both OVs, the EV steers back into the original lane, overtaking the green OV. The EV  maintains the compliance with   the actuation and safety constraints throughout the maneuvers.  
 Table~\ref{Exp-Table: Cost-Comp} offers a quantitative performance evaluation.  Overall, the MPIC-X algorithm demonstrates high computational efficiency and scalability with respect to both the prediction horizon length and the number of particles. The algorithm's cost performance is also excellent,  with just ten particles sufficient to achieve low enough costs at fast computation. These results are consistent with what Table~\ref{Table: Computation-Comp} shows when considering the complexity of the neural network architecture, showing the potential effectiveness of the MPIC-X algorithm for real-world problems. 

\begin{figure*}[t]
    \centering
    \includegraphics[width=0.98\textwidth, trim={5cm 10cm 4cm 8.5cm},clip]{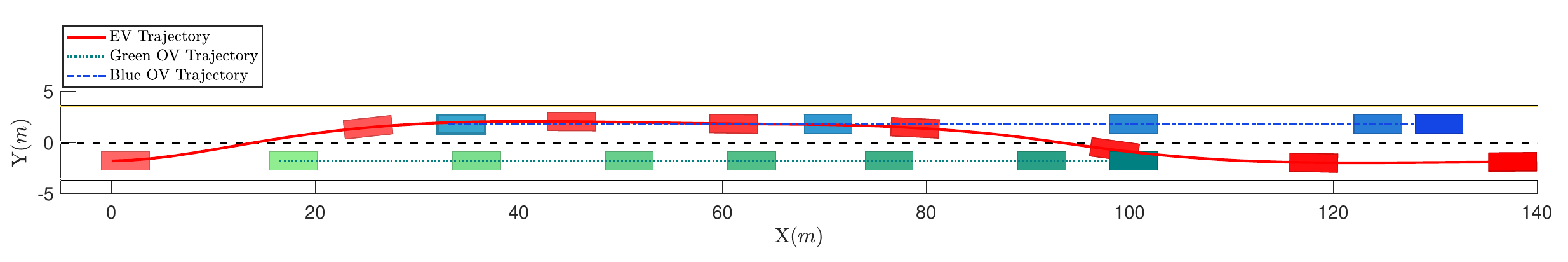}
    \caption{Trajectories of the EV (in red and based on the DevBot 2.0 car) and OVs (in green and blue) in the overtaking scenario. }% The color change from light to dark indicates the passage of the driving time.}
    \label{fig:Exp_ConjestionTraj}
\end{figure*}

\begin{figure*}[t]
    \centering
    \begin{subfigure}[t]{0.325\textwidth}
    \includegraphics[width=\textwidth,trim={3.8cm 8cm 3.4cm 9.1cm},clip]{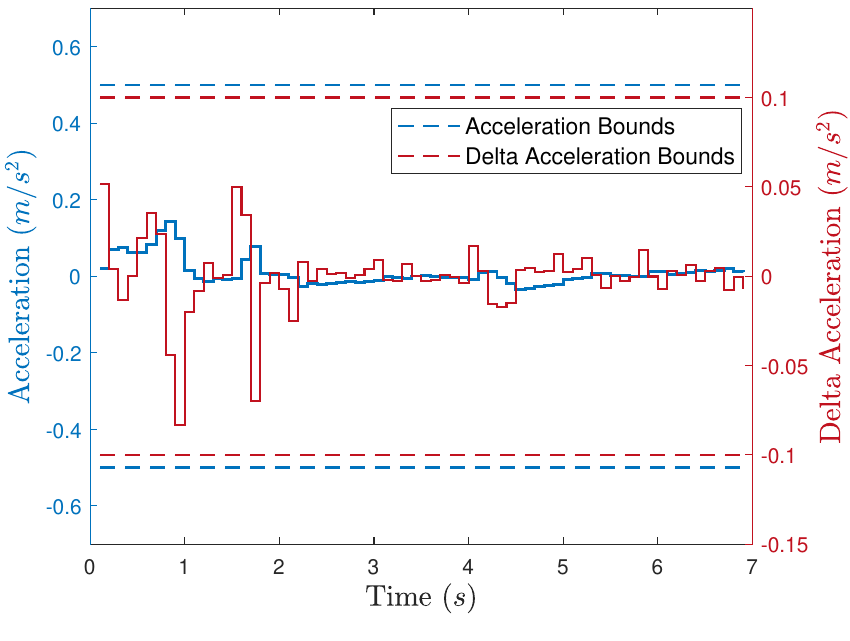}
    \caption{}
    \label{fig:Exp_overtaking-a}
\end{subfigure}
\hspace{1mm}
\begin{subfigure}[t]{0.325\textwidth}
    \includegraphics[width=\textwidth,trim={4.05cm 8cm 3.6cm 9.1cm},clip]{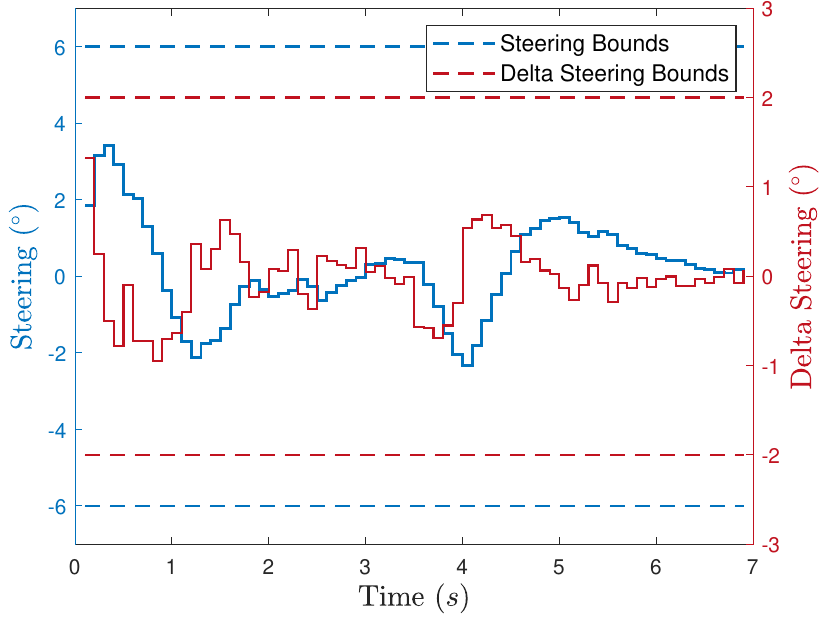}
    \caption{}
    \label{fig:Exp_overtaking-b}
\end{subfigure}
\hspace{1mm}
\begin{subfigure}[t]{0.31\textwidth}
    \includegraphics[width=\textwidth,trim={4cm 8cm 4.3cm 9.05cm},clip]{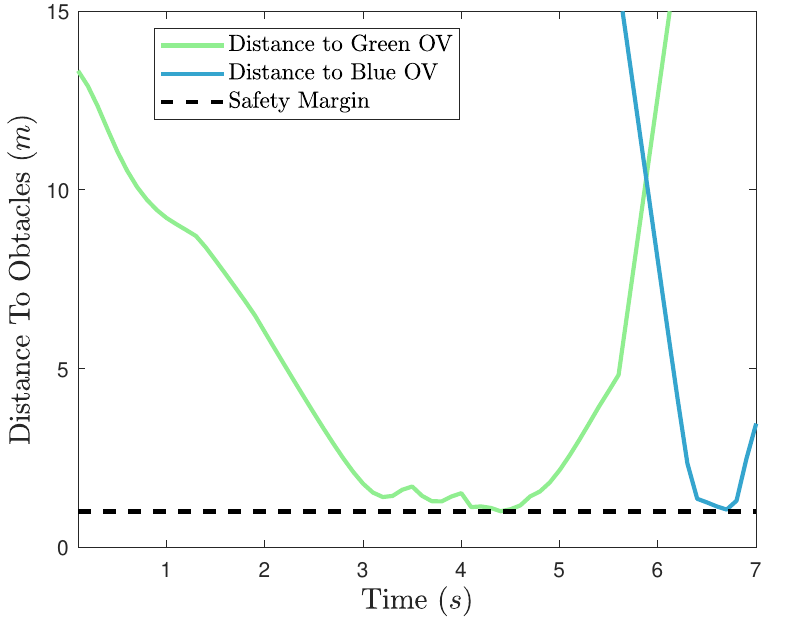}
    \caption{}
    \label{fig:Exp_overtaking-c}
\end{subfigure}

    \caption{Behavior of the DevBot 2.0 car in the overtaking scenario: (a) the acceleration and incremental acceleration profiles in solid curves, with respective bounds in dashed lines; (b) the steering and incremental steering profiles in solid lines, with their respective bounds in dashed lines; (c) the distance between the EV and OVs, with the dashed-line safe margin.}
    \label{fig:Exp_OvertakingControls}
\end{figure*}

\begin{table}[!htbp]
\caption{Performance of the MPIC-X algorithm for the DevBot 2.0 car}

\begin{tabular}{c  c  c  c}
\toprule
\makecell[c]{Horizon ($H$)} & Method & Total Cost & \makecell[c]{Average \\ Computation \\ Time (s)}  \\
\midrule
\multirow{3}{*}{$10$}
 & MPIC-X w. $10$ particles & $4,662$ & $0.034$ \\
 & MPIC-X w. $50$ particles & $4,538$ & $0.147$ \\
 & MPIC-X w. $80$ particles & $4,589$ & $0.217$ \\
\cmidrule(l){1-4}
\multirow{3}{*}{$20$}
 & MPIC-X w. $10$ particles & $4,546$ & $0.102$ \\
 & MPIC-X w. $50$ particles & $4,551$ & $0.469$ \\
 & MPIC-X w. $80$ particles & $4,539$ & $0.659$ \\
\cmidrule(l){1-4}
\multirow{3}{*}{$40$}
 & MPIC-X w. $10$ particles & $4,852$ & $0.256$ \\
 & MPIC-X w. $50$ particles & $4,737$ & $1.108$ \\
 & MPIC-X w. $80$ particles & $4,771$ & $1.502$ \\
\cmidrule(l){1-4}
\multirow{3}{*}{$60$}
 & MPIC-X w. $10$ particles & $4,948$ & $0.430$ \\
 & MPIC-X w. $50$ particles & $4,910$ & $1.826$ \\
 & MPIC-X w. $80$ particles & $4,953$ & $2.784$ \\
\bottomrule
\end{tabular}

\label{Exp-Table: Cost-Comp}
\end{table}

\section{Conclusion}
\label{Sec:Conclusion}
The rise of autonomous driving presents ever-growing demands for better motion planning technologies. MPC has proven to be a useful approach for this application. Meanwhile, machine learning has found its way into vehicle modeling due to its capacity of accurately capturing vehicle dynamics. However, despite the potential for improving motion planning design, machine learning models have been unyielding to MPC, as their strong nonlinearity and nonconvexity resist gradient-based optimization. 

In this paper, we consider the problem of MPC of NSS models and pursue a different perspective---inferring the best control decisions from the control objectives and constraints. This perspective, inspired by the classical control-estimation duality, opens up the avenue for executing MPC through Bayesian estimation and some powerful estimation techniques. Based on this notion, we first reformulate an incremental MPC problem for motion planning into an equivalent Bayesian state smoothing problem. To tackle the problem, we consider particle filtering/smoothing, which, based on sequential Monte Carlo sampling, can handle highly nonlinear systems. This approach, however, often requires large numbers of particles and thus heavy computation to succeed. We then derive and propose implicit particle filtering/smoothing based on banks of unscented Kalman filters/smoothers. This novel approach manages to use much fewer particles in estimation by sampling at highly probable regions of the target distribution to achieve high computational efficiency and accuracy. The resultant framework, called MPIC, and algorithm, called MPIC-X, thus arise out of the development. 

We apply the MPIC-X algorithm to highway driving scenarios via extensive simulations. The simulation results validate the capability of the MPIC-X algorithm in dealing with the control of NSS models for motion planning. Its computation is very fast and almost insensitive to neural network architectures, while well applicable and scalable to long prediction horizons, compared to gradient-based MPC. The proposed framework and algorithm hold a potential for addressing various other robotics and engineering problems that involve the control of machine learning models.

%\section*{References}
%\vspace{-1em}
\bibliographystyle{IEEEtran}
\balance
{\footnotesize
\bibliography{ReferenceList}}

\begin{IEEEbiography}[{\includegraphics[width=1in,height=1.25in,clip,keepaspectratio]{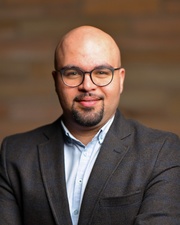}}]{Iman Askari} (Graduate Student Member, IEEE) received the B.S. degree in 2018 and is currently working towards a Ph.D. degree, both  in Mechanical Engineering and at the University of Kansas, Lawrence, KS, USA. His research interests lie at the intersections of control theory,  estimation theory, and machine learning for robotics applications. \end{IEEEbiography}

\begin{IEEEbiography}[{\includegraphics[width=1in,height=1.25in,clip,keepaspectratio]{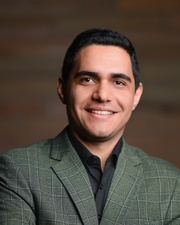}}]{Ali Vaziri} (Graduate Student Member, IEEE) received the B.S. degree in Marine Engineering from Sharif University of Technology, Tehran, Iran, in 2021, where he conducted research on motion planning for surface marine robots. He is currently pursuing a Ph.D. degree in Mechanical Engineering at the University of Kansas, Lawrence, KS, USA. His current research interests include estimation theory and optimal control of high-dimensional data-driven dynamical models. \end{IEEEbiography}

\begin{IEEEbiography}[{\includegraphics[width=1in,height=1.25in,clip,keepaspectratio]{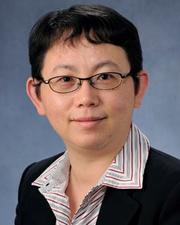}}]{Xumein Tu} received the B.S. degree in Computer Sciences from Beijing Normal University, China in 1997, the M.S. degree in Applied Mathematics from Worcester Polytechnic Institute in 2002, and the Ph.D. degree in Mathematics in 2006  from New York  University.  She is currently a Professor at the Department of Mathematics, University of Kansas, Lawrence, KS, USA. Her main research interests are numerical partial differential equations and data assimilations.  \end{IEEEbiography}
\begin{IEEEbiography}[{\includegraphics[width=1in,height=1.25in,clip,keepaspectratio]{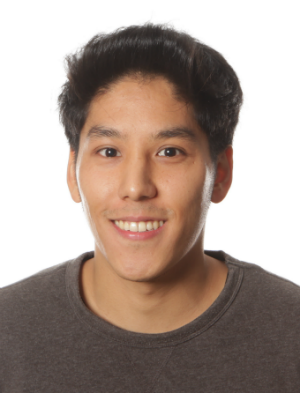}}]{Shen Zeng} (Member, IEEE) studied Engineering Cybernetics, Mechatronics, and Mathematics at the University of Stuttgart, from where he also received a Ph.D. degree in 2016. He is currently an Assistant Professor at Washington University in St. Louis in the Department of Electrical and Systems Engineering. His research interests are in systems theory, and, more generally, applied mathematics. \end{IEEEbiography}
\begin{IEEEbiography}[{\includegraphics[width=1in,height=1.25in,clip,keepaspectratio]{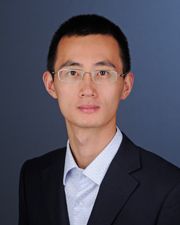}}]{Huazhen Fang}
(Member, IEEE) received the B.Eng. degree in Computer Science and Technology from Northwestern Polytechnic University, Xi’an, China, in 2006, the M.Sc. degree from the University of Saskatchewan, Saskatoon, SK, Canada, in 2009, and the Ph.D. degree from the University of California, San Diego, CA, USA, in 2014, both in mechanical engineering. He is currently an Associate Professor of Mechanical Engineering with the University of Kansas, Lawrence, KS, USA. His research interests include control and estimation theory with application to energy management and robotics. He was a recipient of the 2019 National Science Foundation CAREER Award. He currently serves as an Associate Editor for IEEE Transactions on Industrial Electronics and IEEE Control Systems Letters. \end{IEEEbiography}

\end{document}